\documentclass[runningheads]{llncs}

\usepackage{eccv}

\usepackage[a4paper,margin=1in]{geometry}

\usepackage{eccvabbrv}

\usepackage{graphicx}
\usepackage{booktabs}
\usepackage{multirow}
\usepackage{wrapfig}
\usepackage{amssymb}
\usepackage{pifont}
\usepackage{float}
\usepackage{capt-of}
\usepackage{subcaption}
\usepackage{algorithm}
\usepackage{algorithmic}
\usepackage{tikz}
\usetikzlibrary{shapes.geometric, arrows.meta, positioning, fit, calc}
\usepackage{listings}
\usepackage{xcolor}
\usepackage{amsmath}
\usepackage{docmute}

\definecolor{codegreen}{rgb}{0,0.6,0}
\definecolor{codegray}{rgb}{0.5,0.5,0.5}
\definecolor{codepurple}{rgb}{0.58,0,0.82}
\definecolor{backcolour}{rgb}{0.98,0.98,0.98}
\definecolor{keywordcolor}{rgb}{0.6,0,0.6}
\definecolor{commentcolor}{rgb}{0,0.5,0}
\definecolor{stringcolor}{rgb}{0.8,0.2,0.2}

\lstdefinestyle{pythonstyle}{
    backgroundcolor=\color{backcolour},
    commentstyle=\color{commentcolor}\itshape,
    keywordstyle=\color{keywordcolor}\bfseries,
    numberstyle=\tiny\color{codegray},
    stringstyle=\color{stringcolor},
    basicstyle=\ttfamily\scriptsize,
    breakatwhitespace=false,
    breaklines=true,
    captionpos=b,
    keepspaces=true,
    numbers=none,
    numbersep=5pt,
    showspaces=false,
    showstringspaces=false,
    showtabs=false,
    tabsize=2,
    xleftmargin=3pt,
    xrightmargin=3pt,
    frame=none,
    morekeywords={def, return, for, in, if, with, as}
}
\tikzset{
    process/.style={rectangle, minimum width=3cm, minimum height=0.8cm,
        text centered, draw=black, text width=5.5cm, font=\small},
    arrow/.style={-{Stealth[length=2mm]}, thick}
}

\usepackage[accsupp]{axessibility}  

\newcommand{\cmark}{\ding{51}}%
\newcommand{\xmark}{\ding{55}}%

\usepackage{hyperref}

\usepackage{orcidlink}

\begin{document}

\title{iADD: Improving Alignment and Diversity in Diffusion Policy Optimization} 

\titlerunning{i-ADD: Improving Alignment and Diversity in Diffusion Policy Optimization}

\author{Ashok Prasad Neupane\inst{1}\textsuperscript{*} \and
Saugat Adhikari\inst{1}\textsuperscript{*} \and
Pramish Paudel\inst{3}\textsuperscript{*} \and
Ajad Chhatkuli\inst{2,3} \and
Danda Pani Paudel\inst{2,3}}

\authorrunning{A.~P.~Neupane et al.}

\institute{Pulchowk Campus, IOE, Tribhuvan University, Lalitpur, Nepal \and
NAAMII, Kathmandu, Nepal \and
INSAIT, Sofia University ``St. Kliment Ohridski'', Bulgaria\\
\textsuperscript{*}Equal contribution}

\maketitle
\begin{figure}[H]
    \centering
    \begin{subfigure}[b]{0.54\textwidth}
        \centering
        \includegraphics[width=\textwidth]{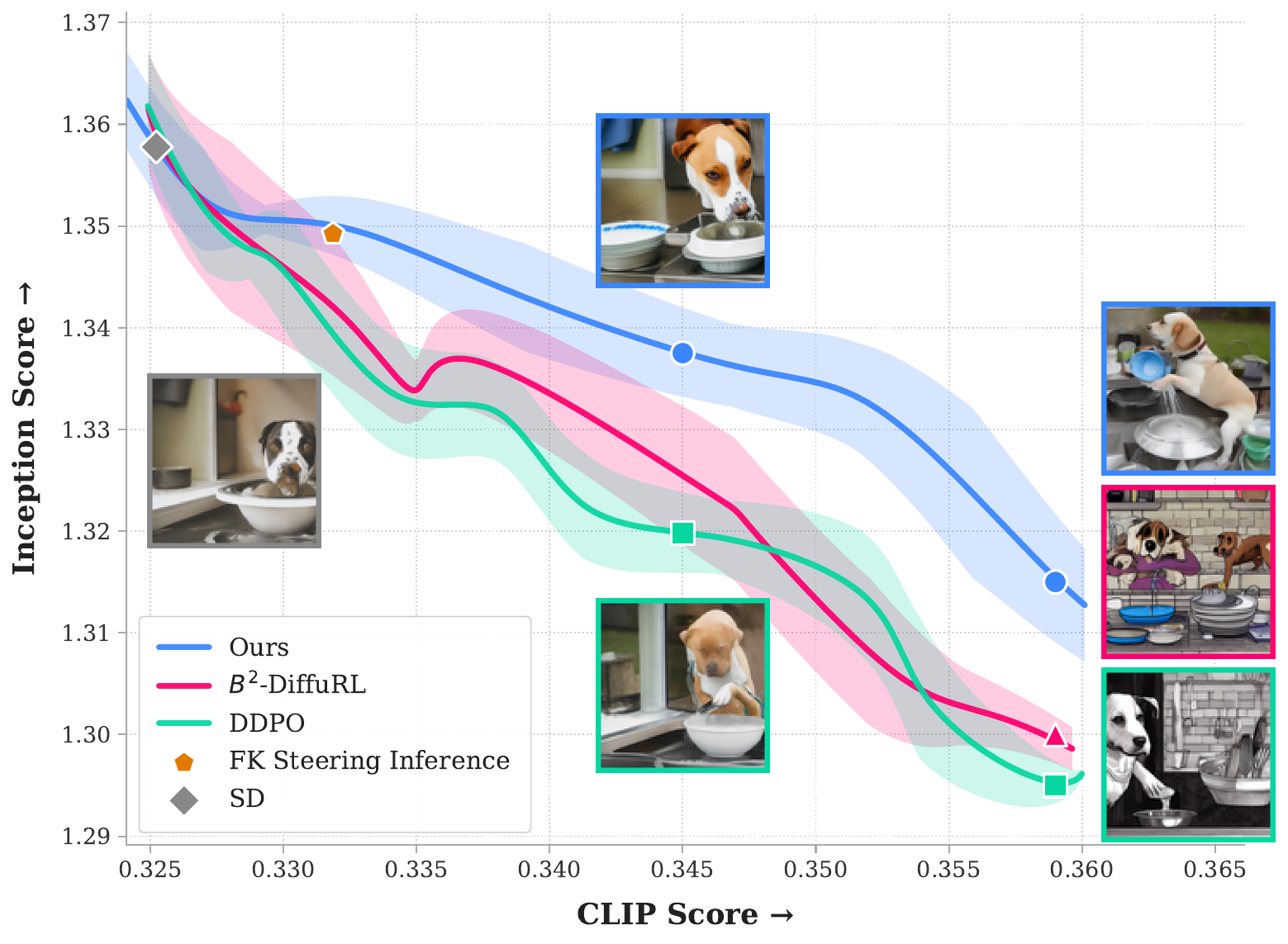}
        \caption{Alignment-diversity trade-off curves}
        \label{fig:diversity_vs_reward}
    \end{subfigure}
    \hfill
    \begin{subfigure}[b]{0.45\textwidth}
        \centering
        \includegraphics[width=\textwidth,trim={0 40 350 20},clip]{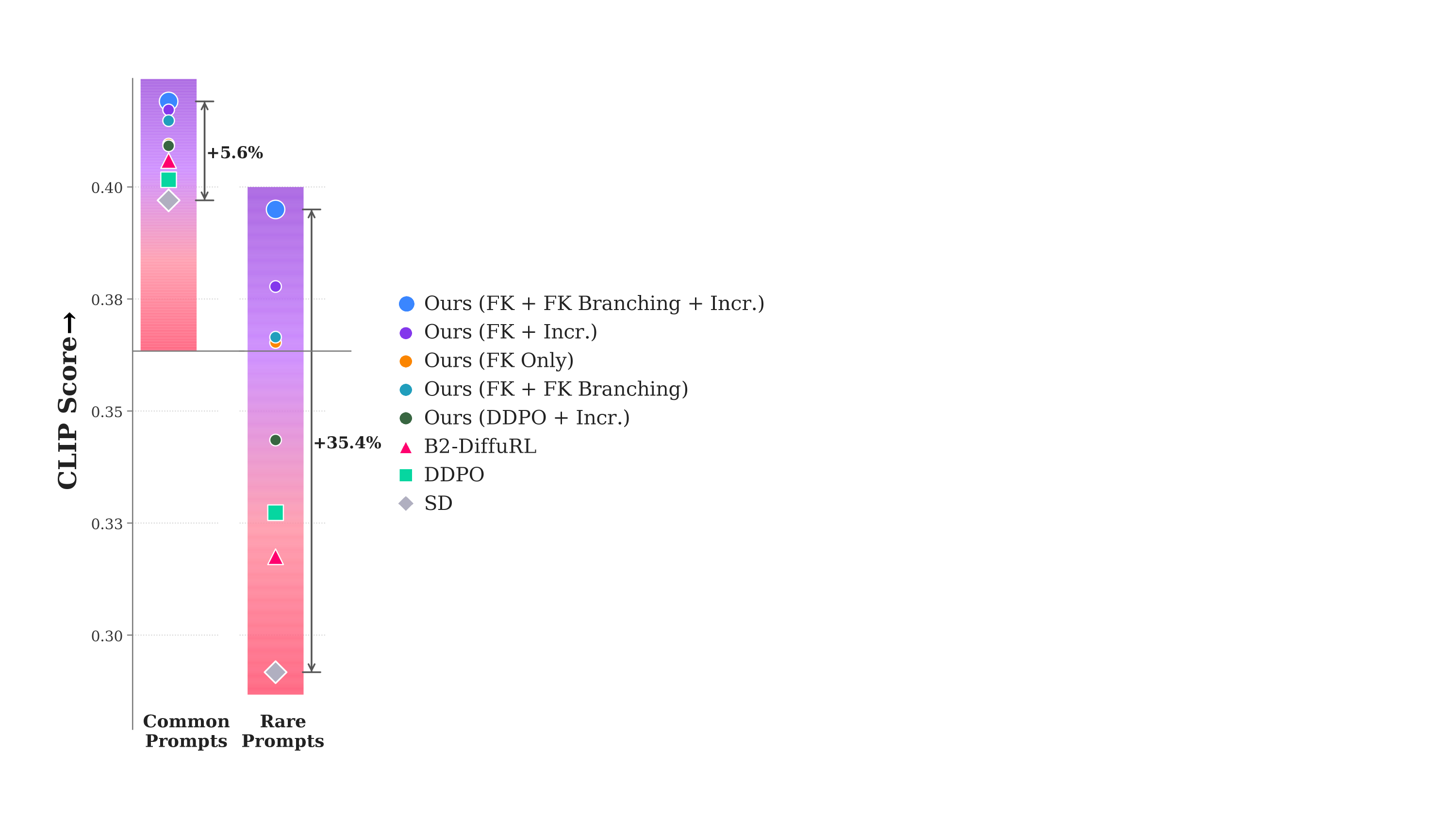}
        \caption{Alignment performance in common and rare prompts}
        \label{fig:common_rare_barplot}
    \end{subfigure}
    \caption{\textbf{Performance in alignment and diversity.} \textbf{(a)} Our method, iADD, achieves a higher CLIP Score at any given diversity level compared to other methods. Notably, as reward increases, competing methods (DDPO and $B^2$-DiffuRL) produce increasingly cartoonish, over-saturated outputs, whereas our method preserves photorealistic quality throughout training. The visualized images are for the prompt ``a dog washing dishes''. \textbf{(b)} Alignment performance on common and rare prompts, evaluated at constant Inception Score(IS) for each method. Rare prompts represent conditions that inherently yield low-reward outcomes under the pretrained model.}
    \label{fig:teaser-figure}
\end{figure}

\begin{abstract}
  Reinforcement learning based post training of diffusion models, such as Denoising Diffusion Policy Optimization (DDPO), optimizes a reverse diffusion process under a reward function. However, current approaches to reward optimizations do so at the cost of diversity and quality. In this paper, we provide better tradeoffs through careful theoretical considerations and method design. We analyze the theoretical framework and mathematically demonstrate that \emph{only-latter timestep} updates of diffusion model may be harmful for diversity contrary to the conclusions presented in a previous work. Additionally, we propose an incremental Feynman-Kac training based on strong theoretical foundations in order to achieve the best-yet alignment-diversity tradeoffs. We perform extensive experiments and compare our method against related diffusion policy optimization approaches in three different tasks and also provide strong ablations for each component, thus validating strong performance gains in both alignment and diversity.
  

  Project page: \url{https://saugat2002.github.io/iadd/} \textperiodcentered{} Code: \url{https://github.com/Saugat2002/iADD}.

\end{abstract}

\section{Introduction}

Diffusion Models are currently the gold standard for generation and editing tasks on images~\cite{ddpm, ddim, flow, rombach2022highresolution,song2021scorebasedgenerativemodelingstochastic, sohl2015deep,nichol2022glidephotorealisticimagegeneration, saharia2022photorealistic}, videos~\cite{ho2022videodiffusionmodels, ho2022imagenvideohighdefinition, Chen2024forcing} as well as 3D scenes~\cite{gokmen2025inferring, hu2026mixed, tang2024diffuscene} or objects. Through time-dependent iterative movement defining a path, diffusion methods learn to reach the target data distribution, exhibiting strong generalization and diversity~\cite{Dhariwal2021,Li2023Generalization, nichol2022glidephotorealisticimagegeneration}. Nevertheless, the quality of the modeling is limited by the target distribution samples; thus in most cases generation quality is naturally impacted for rare classes and specific user constraints. Reinforcement Learning (RL)~\cite{Williams1992REINFORCE, ppo, instructgpt, schulman2015trust} has the potential to address the problem, without using additional data, in the ideals of learning solely through experience~\cite{Silver2025}.

Recent works~\cite{black2023ddpo,fan2023dpok, Zhao2025Score_as_asction,hu2025towards, alignprop, d3po, xu2023imagereward} have precisely addressed the problem of finetuning the reverse diffusion process online, by applying the RL algorithm~\cite{Williams1992REINFORCE}. However, naive policy optimization of diffusion models can lead to mode collapse and general quality degradation of the generated samples~\cite{fan2023dpok, hu2025towards}. This is reflected in the \emph{alignment-diversity tradeoffs}, see Fig.~\ref{fig:teaser-figure}.

Circumventing mode collapse and general quality degradation in lack of additional preference data~\cite{wallace2023diffusiondpo,li2024diffusionkto,zhang2024largescalerl,kang2025rethinkingdpo}, though important is highly challenging. Thus, most approaches tend to be domain specific, and thus are validated on specific tasks. A more generic and theoretically grounded study of the alignment-diversity tradeoffs under the constraints of additional preference data-free rewards, remains under-explored. Existing solutions often perform regularization, \eg through selected timestep updates~\cite{barcelo2024avoiding,hu2025towards} or via measuring the KL divergence~\cite{fan2023dpok}. Specifically, \cite{hu2025towards} intuitively argues that reward backpropagation is more reliable for latter timesteps and thus do not update the model on earlier timesteps. On the other hand, \cite{barcelo2024avoiding, meng2021sdedit} refine middle timesteps based on the observation of how semantics are constructed in diffusion models~\cite{Dhariwal2021,Choi2022}. Similarly, \cite{fan2023dpok,Zhao2025Score_as_asction} specifically consider the Denoising Diffusion Probabilistic Model (DDPM) in order to propose reward optimization strategies. Furthermore, \cite{Zhao2025Score_as_asction} proposes a continuous time formulation in DDPM for RL policy optimization.

In this paper, we study the alignment-diversity tradeoff in finetuning discrete-time denoising diffusion model through early intervention as well as careful regularization. We formulate the tradeoff problem as that of attaining the minimum change in the target samples required for the alignment, while preserving the original diversity through structural regularization, and target distribution's volume preservation. We discover that attaining higher alignment with minimal diversity tradeoff requires early intervention via two processes. The first and retrospectively obvious solution is to update the early timesteps of the reverse denoising path, contrary to \cite{hu2025towards} that updates the late timesteps. We provide actionable theoretical results that untangles the puzzle of timestep selection. The results justify partial timestep selection from the viewpoint of target distribution's volume preservation. Moreover, the theoretical analyses suggest an approach, which incrementally increases the per-iteration sparse timestep updates, over the course of the finetuning. Additionally, our second solution is to perform early intervention by looking ahead in the denoising path. Specifically, we propose the discrete-time implementation of the Feynman-Kac path guidance~\cite{Moral2004,Kappen2005,Singhal2025FK_Steering,Skreta2025FK_Correctors, smith2013sequential, jasra2009sequential} during training in the following two ways: discontinuing least promising paths with low reward expectations to improve the alignment and branching the promising ones in order to encourage diversity. Table~\ref{tab:component_disassembly} shows the relative impact and relevant theory behind the method's components.


\begin{figure}[h]
    \centering
    \vspace{-1.2\baselineskip}
    \includegraphics[width=\linewidth, trim={0 40 0 50}, clip]{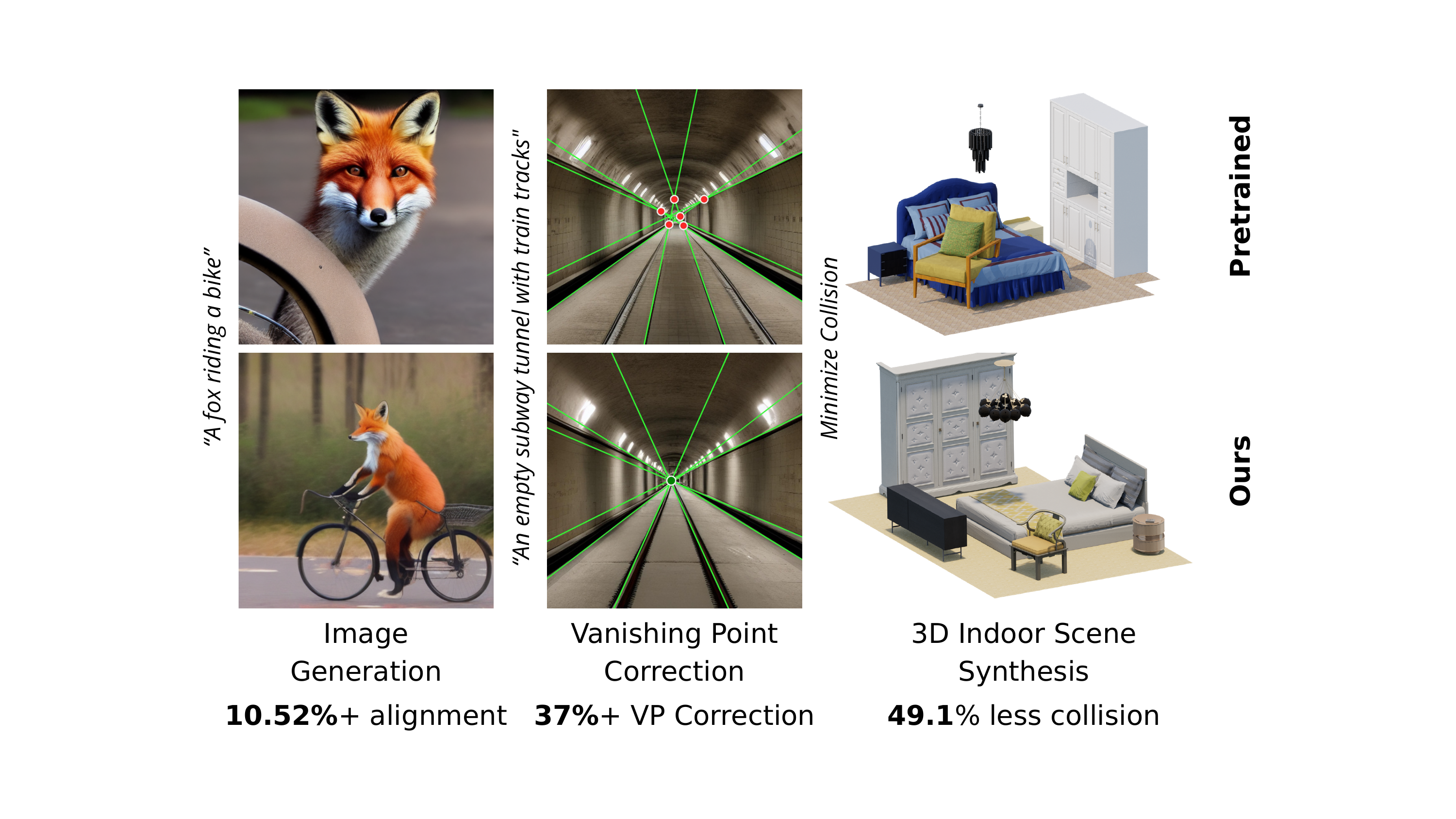}
    \vspace{-1.8\baselineskip}
    \caption{\textbf{Versatility across diverse generation tasks.}}
    \label{fig:intro_figure}
    \vspace{-0.8\baselineskip}
\end{figure}
\vspace{-0.5\baselineskip}
Our final method, termed, iADD thus employs both early timestep update, early intervention as well as structural regularization in order to obtain a balanced approach to reward-based finetuning of diffusion models. We perform experiments on three different tasks, rare prompt-based image generation, vanishing point correction in image generation and 3D scene generation, in order to show the versatility. Our comparisons show clear improvements in alignment as well as diversity/quality of generation in all three tasks compared to the state-of-the-art approaches.
In summary, our main contributions are as follows:
\begin{enumerate}
    \item[a.] We provide theoretical analysis of the impact of timestep selection strategies on diversity.
    \item[b.] We propose iADD, a novel training method that integrates an incremental timestep curriculum with discrete-time Feynman-Kac (FK) branching to optimally balance reward alignment and diversity.
    \item[c.] Extensive experiments across image and 3D scene generation demonstrate that iADD improves alignment, diversity, and rare-event success; comparisons with recent GRPO and inference-time guidance methods further show that its components are complementary to the underlying policy optimizer.
\end{enumerate}


\begin{table}[h]
\centering
\caption{Disassembly of our Method's components and their roles. Up arrow ($\uparrow$) denotes improvement and down arrow ($\downarrow$) denotes worsening. Additional reference to sections or theoretical justifications are mentioned in the most relevant column.}
\label{tab:component_disassembly}
{\renewcommand{\arraystretch}{1.12}
\begin{tabular}{@{}clcccc@{}}
\toprule
& \textbf{Components} & \textbf{Alignment} & \textbf{Diversity} & \textbf{Rarity} & \textbf{Training Eff.} \\ \midrule
\multirow{3}{*}{\rotatebox[origin=c]{90}{\scriptsize Timesteps}} & Early timesteps & $\uparrow$  & $\uparrow$, Prop.\ \ref{prop:control_authority} & $\uparrow$ & $\uparrow$ \\
& Timestep sampling & $\uparrow$ & $\uparrow$, Prop.\ \ref{prop:control_authority}, \ref{prop:any_vs_all} & $\uparrow$ & $\uparrow$ \\
& Incremental timesteps & $\uparrow$ & $\uparrow$, Remark.\ \ref{rem:incremental} & $\uparrow$ & $\uparrow$ \\
\multicolumn{6}{@{}c@{}}{\dotfill} \\
\multirow{2}{*}{\rotatebox[origin=c]{90}{\scriptsize FK}} & FK Guidance                  & $\uparrow$, Sec.~\ref{sec:fk}  & $\uparrow$  & $\uparrow \uparrow$ & $\downarrow$  \\
& FK Branching                 & $\uparrow$ & $\uparrow \uparrow$, Sec.~\ref{sec:fk} & $\uparrow$  & $\downarrow$ \\ \bottomrule
\end{tabular}
}
\end{table}

\vspace{-1em}
\section{Related Work}
\vspace{-1em}
\emph{Inference-time adjustments.}
In order to improve alignment-diversity, reward metrics are also used solely at inference time. Such training free-approaches \cite{singhal2025general, Skreta2025FK_Correctors,song2021scorebasedgenerativemodelingstochastic} utilize inference time measures to maximize the reward for specific generated samples. However, these techniques are also susceptible to degradation of sample quality at higher reward and moreover do not scale well during inference.

\noindent\emph{Preference-based optimization.} RL when used with additional preference data allows for advanced techniques that can preserve diversity. Such approaches are the counterparts of RL with Human Feedback (RLHF)~\cite{Christiano2017}, popularly used in Large Language Models finetuning~\cite{Ziegler2019}, in the diffusion policy optimization. Methods in the category primarily follow the Direct Policy Optimization (DPO)~\cite{rafailov2024directpreferenceoptimizationlanguage , wallace2023diffusiondpo} strategies either in discrete-time formulation~\cite{wallace2023diffusiondpo} or continuous-time~\cite{Hua2025neurips}. However, they are severely limited by the requirement of human preference data, requiring an offline strategy. 

\noindent\emph{Online RL for Diffusion Models.} Our work is on online RL finetuning of diffusion models~\cite{black2023ddpo,fan2023dpok,Behjoo2025,hu2025towards,Yan2025entropy,Clark2023}, which has seen limited applications in diffusion models, largely due to the difficulty of the joint alignment-diversity optimization. In the topic, DDPO~\cite{black2023ddpo} is a seminal work that finetunes the complete trajectory as discrete steps in the Markov chain. DPOK~\cite{fan2023dpok} proposed a KL-divergence regularization using the DDPM formulation. Notably, \cite{hu2025towards} shows the challenges of the sparse terminal rewards~\cite{andrychowicz2018hindsightexperiencereplay,hare2019dealingsparserewardsreinforcement,memarian2021selfsupervisedonlinerewardshaping,nachum2018dataefficienthierarchicalreinforcementlearning,devidze2022explorationguided,gupta2022unpackingrewardshapingunderstanding,hu2020learningutilizeshapingrewards,Ng1999PolicyIU} and suggests solely updating the later timesteps of denoising, while also incorporating contrastive pairs in branches for improving diversity. A recent work also formulates the continuous-time online RL for diffusion models~\cite{Zhao2025Score_as_asction} in the DDPM formulation while also defining each path's value function at any given time similar to a Feynman-Kac potential. Other approaches~\cite{Clark2023,Prabhudesai2024} use task-specific considerations for online finetuning of diffusion models. However, current approaches do not treat time-step selection under careful theoretical considerations and under-exploit the denoising paths in reward optimization leading to poorer tradeoffs.


\begin{table}[t]
    \centering
    \caption{Summary of methods and techniques across training, inference, and diffusion formulation. Our method is designed for the DDIM formulation.}
    \label{tab:related-work-summary}
    {\setlength{\tabcolsep}{4pt}
    \renewcommand{\arraystretch}{1.12}
    \resizebox{\textwidth}{!}{%
    \begin{tabular}{l|ccc|c|c}
        \hline
        \multicolumn{1}{c|}{\multirow{2}{*}{\textbf{Method}}} & \multicolumn{3}{c|}{\textbf{Training}} & \multicolumn{1}{c|}{\textbf{Inference}} & \textbf{Formulation} \\
        \cline{2-6}
        & \textbf{Incr.} & \textbf{FK} & \textbf{Branch.} & \textbf{No Test-time Scaling} & \textbf{DDPM/DDIM} \\
        \hline
        DPOK\cite{fan2023dpok} & \xmark & \xmark & \xmark & \cmark & DDPM \\
        Score-as-Action\cite{Zhao2025Score_as_asction} & \xmark & \cmark & \xmark & \cmark & DDPM \\
        \hline
        \hline
        DDPO\cite{black2023ddpo} & \xmark & \xmark & \xmark & \cmark & DDIM \\
        B$^2$-DiffuRL\cite{hu2025towards} & \xmark & \xmark & \cmark & \cmark & DDIM \\
        FK Steering \& FK Correctors\cite{singhal2025general, Skreta2025FK_Correctors} & \xmark & \xmark & \xmark & \xmark & DDIM \\
        \hline
        Ours & \cmark & \cmark & \cmark & \cmark & DDIM \\
        \hline
    \end{tabular}
    }}
\end{table}
\vspace{-0.3em}
\section{Preliminaries}
\vspace{-0.3em}
\paragraph{Discrete-Time Diffusion Models.}
\label{subsec:diffusion}

Diffusion models~\cite{sohl2015deep} are a class of latent variable models that learn to approximate a target data distribution by systematically reversing a predefined Markovian corruption process. Our work specifically builds upon discrete-time formulations, where the generation trajectory unfolds over a finite set of steps $t \in \{0, \dots, T\}$.

\noindent\textbf{Forward Process.}
The forward process progressively corrupts a data sample $x_0 \sim q(x_0)$ by injecting Gaussian noise according to a fixed variance schedule $\beta_1, \dots, \beta_T$. This trajectory is defined as a Markov chain:
\begin{equation}
q(x_{1:T} \mid x_0) = \prod_{t=1}^T q(x_t \mid x_{t-1}),
\label{eq:forward_process}
\end{equation}
where the transition dynamics are given by $q(x_t \mid x_{t-1}) = \mathcal{N}(x_t; \sqrt{1-\beta_t} \, x_{t-1}, \beta_t I)$. For a sufficiently large number of steps $T$, the final state $x_T$ converges to an isotropic Gaussian noise distribution $\mathcal{N}(0, I)$.

\noindent\textbf{Reverse Generative Process.}
The core objective of a diffusion model is to learn the reverse of this trajectory, constructing clean data samples starting from pure noise $x_T$. The reverse process is modeled as a parameterized Markov chain, often conditioned on auxiliary context $c$ (\eg, text prompts):
\begin{equation}
p_\theta(x_{0:T} \mid c) = p(x_T) \prod_{t=1}^T \pi_\theta(x_{t-1} \mid x_t, c),
\label{eq:reverse_process}
\end{equation}
where $\pi_\theta(x_{t-1} \mid x_t, c)$ is approximated by a deep neural network predicting the parameters of the Gaussian transition necessary to denoise $x_t$ into $x_{t-1}$.

\paragraph{Denoising Diffusion Implicit Models (DDIM)} 
DDIM~\cite{ddim} generalizes the standard DDPM~\cite{ddpm} forward process to a family of non-Markovian diffusion processes that share the same training objective. This decouples the number of inference steps from the number of training timesteps, enabling sampling over an arbitrary subsequence $\mathcal{T}_{sub} \subset \{1, \ldots, T\}$ without any retraining. Crucially, at each denoising step $t$, DDIM produces an explicit estimate of the clean sample $x_0$ via:

\begin{equation}
\hat{x}_0 = \frac{x_t - \sqrt{1 - \alpha_t}\, \epsilon_\theta^{(t)}(x_t)}{\sqrt{\alpha_t}}
\end{equation}

which we leverage to compute reward estimates from intermediate timesteps, enabling us to guide the generative trajectory before the final sample is produced.

\paragraph{Denoising Diffusion Policy Optimization (DDPO)}
\label{subsec:ddpo}

Denoising Diffusion Policy Optimization (DDPO)~\cite{black2023ddpo} treats the iterative reverse diffusion process as a multi-step Markov Decision Process (MDP). At each timestep $t$, the state is defined as $s_t \triangleq (c, t, x_t)$, the action as $a_t \triangleq x_{t-1}$, and the policy is parameterized by the reverse transition probability $\pi(a_t \mid s_t) \triangleq \pi_\theta(x_{t-1} \mid x_t, c)$.

The objective is to optimize a downstream reward $R(x_0, c)$ based on the final generated sample $x_0$ and the conditioning context $c$. Because each reverse transition is an isotropic Gaussian, exact per-step log-likelihoods are tractable. DDPO optimizes this MDP using an importance-weighted policy gradient estimator (DDPO-IS), allowing for multiple gradient updates using data from a previous policy parameters $\theta_{\text{old}}$:
\begin{equation}
\nabla_\theta \mathcal{F}(\theta) = \mathbb{E} \left[ \sum_{t=0}^T \frac{\pi_\theta(x_{t-1} \mid x_t, c)}{\pi_{\theta_{\text{old}}}(x_{t-1} \mid x_t, c)} \nabla_\theta \log \pi_\theta(x_{t-1} \mid x_t, c) \cdot R(x_0, c) \right].
\label{eq:ddpo_is}
\end{equation}

However, standard DDPO updates all timesteps equally and relies solely on a single final reward evaluated on the complete trajectory. This sparse terminal reward and uniform credit assignment leads to reward hacking, which naturally motivates our proposed particle-based local advantage estimation and incremental training scheduling.

\section{Method}
\begin{figure}[h]
    \centering
    \includegraphics[width=1\textwidth]{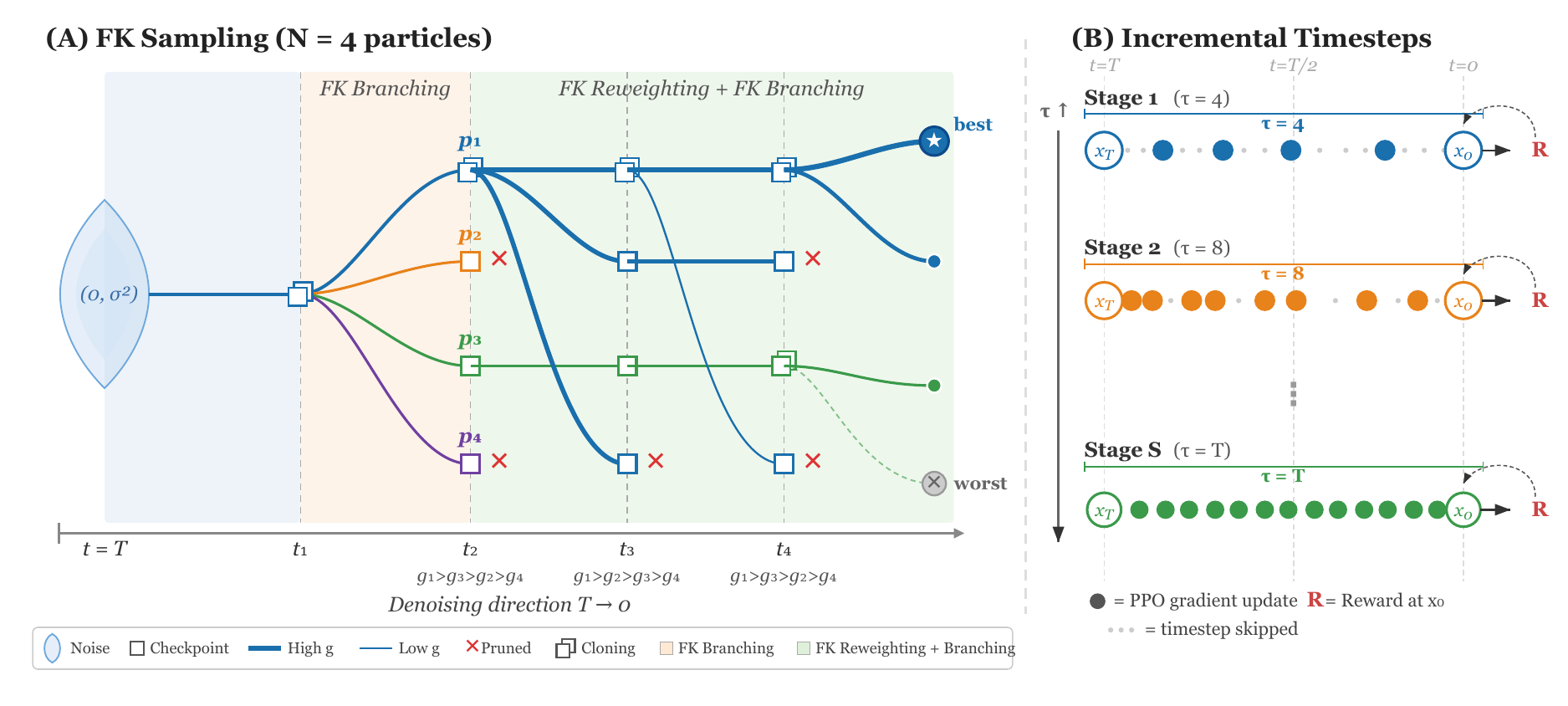}
    \caption{\textbf{Overview of FK Sampling and Incremental Timestep Training.} \textbf{(A)} FK Sampling: Our approach employs a Feynman-Kac (FK) particle filter for sampling. Starting from Gaussian noise at $t=T$, trajectories ($p_1, \dots, p_4$) are evaluated at discrete checkpoints $t_n$ based on their potentials $g$. The $FK Branching$ mechanism duplicates high-potential particles while pruning low-potential ones, effectively redirecting the stochastic denoising paths toward regions of higher expected reward. \textbf{(B)} Incremental timestep updates: During training, we utilize a progressive stage-wise strategy where the cardinality of the updated timesteps $\mathcal{S}$ increases.}
    \label{fig:method}
\end{figure}
\vspace{-0.5em}

\subsection{Problem formulation}
Let $\pi_\phi(\mathbf{x}_{t-1}\mid \mathbf{x}_t, c)$ denote a pretrained conditional diffusion \emph{reverse} transition model (with parameters $\phi$) used for sampling, where $c$ is a condition (e.g., a text prompt) and $\mathbf{x}_T\sim \mathcal{N}(\mathbf{0},\mathbf{I})$. Sampling produces a trajectory $\tau=(\mathbf{x}_T,\mathbf{x}_{T-1},\ldots,\mathbf{x}_0)$ by iterating $\mathbf{x}_{t-1}\sim \pi_\phi(\cdot\mid \mathbf{x}_t,c)$ for $t=T,\ldots,1$. This induces a marginal distribution over outputs $\mathbf{x}_0$, which we denote by $p_\phi(\mathbf{x}_0\mid c)$.

Our goal is to adapt the sampler to align generated samples with user preferences while preserving the original prior and diversity. We formalize preferences via a reward function $R(\mathbf{x}_0, c)\in\mathbb{R}$ that scores the final generated sample (e.g., aesthetics, constraint satisfaction, or a task-specific metric), and seek to increase expected reward while preserving sample quality and diversity.

We parameterize the \emph{trainable} sampler by $\pi_\theta(\mathbf{x}_{t-1}\mid \mathbf{x}_t, c)$, initialized from the pretrained model ($\theta\leftarrow \phi$) and optimized using reinforcement learning (DDPO) to maximize a sparse terminal reward:
\begin{equation}
    \mathcal{F}(\theta)\;=\;\mathbb{E}_{\tau\sim \pi_\theta}\big[R(\mathbf{x}_0,c)\big],
\end{equation}
where a trajectory $\tau = (\mathbf{x}_T,\mathbf{x}_{T-1},\ldots,\mathbf{x}_0)$ is generated by repeatedly sampling $\mathbf{x}_{t-1}\sim \pi_\theta(\cdot\mid \mathbf{x}_t,c)$ for $t=T,\ldots,1$.

\subsection{Sparse Incremental Timestep Update}
We use the DDIM formulation and first consider the role of structural regularization via selected timestep updates~\cite{hu2025towards,barcelo2024avoiding}. We denote the list of updated timesteps as $\mathcal{S} = {t_i}$. Below we provide actionable theoretical results followed by specific design principles implied thereof.
The proofs of Propositions \ref{prop:control_authority}, \ref{prop:any_vs_all} and Remark \ref{rem:incremental} are provided in the supplementary.

\begin{proposition}[Generative Reach]
\label{prop:control_authority}
    Let $R(\mathbf{x}_0, c)$ be a sparse terminal reward and $\delta \mathbf{\epsilon}_t$ be the parameter update induced by the RL objective at timestep $t$. For a fixed budget of $N=|\mathcal{S}|$ optimized timesteps, a uniform sampling strategy $\mathcal{S}_{uni} \sim \mathcal{U}(1, T)$ yields a strictly greater total perturbation magnitude $||\delta \mathbf{x}_0||$ than a late-stage strategy $\mathcal{S}_{late} \sim \mathcal{U}(1, N)$.
\end{proposition}

The first-order sensitivity used in the proof is a convenient way to expose the timestep dependence; the comparative conclusion only requires the corresponding Lipschitz sensitivity $\|\mathcal{J}_{t-1\to0}\beta_t\|$ to be larger for earlier, high-leverage timesteps under the diffusion schedule. Proposition~\ref{prop:any_vs_all} similarly isolates the structural effect of sparse updates under idealized optimization. These results therefore establish comparative predictions rather than exact training-time bounds; Sec.~\ref{sec:empirical-theory} tests both predictions under nonlinear fine-tuning dynamics.

\paragraph{Early timestep update.} The above result directly indicates that early timesteps should be updated in the RL finetuning of diffusion models, contrary to the conclusions of previous works~\cite{hu2025towards,barcelo2024avoiding}.

\begin{proposition}[Volume Preservation via Temporal Sparsity]
\label{prop:any_vs_all}
 Let $\mathcal{J}_{T \to 0}$ be the total Jacobian of a $T$-step deterministic reverse diffusion process. For a fixed RL optimization budget, a sparse update strategy over a subset $\mathcal{S}$ (where $|\mathcal{S}| = N \ll T$) preserves the probability volume of the prior distribution strictly better than a full update strategy ($\mathcal{S} = \{1, \dots, T\}$), as measured by the log-determinant of the global Jacobian.
\end{proposition}

\begin{remark}[Capacity and Over-parametrization]
\label{rem:capacity}
Proposition \ref{prop:any_vs_all} can be interpreted through the lens of model capacity. Since distinct diffusion timesteps $\mathcal{S}$ specialize in non-overlapping spectral components of the image manifold \cite{Dhariwal2021}, a full update strategy ($\mathcal{S} = \{1, \dots, T\}$) effectively increases the degrees of freedom available to the RL optimizer. By optimizing this over-parameterized "temporal ensemble," the model more easily finds degenerate solutions that satisfy the terminal reward $R$ at the cost of global prior collapse.
\end{remark}

\paragraph{Sparse timestep update.} The above result clearly indicates that selecting sparse timesteps for update provides strong regularization in policy optimization of the diffusion models. In fact, this also provides a tangible explanation for the improved performance of \cite{hu2025towards} when using only late timesteps with additional components of the approach.

\begin{remark}[Synthesis: Temporal Curriculum of Generative Reach]
\label{rem:incremental}
The relationship between generative reach (Prop. \ref{prop:control_authority}) and volume preservation (Prop. \ref{prop:any_vs_all}) implies an optimal curriculum for reward finetuning: an incremental expansion of the active timestep subset $\mathcal{S}$. By initially restricting optimization to a sparse set of high-leverage early timesteps ($t \sim T$), the model achieves global mode discovery with minimal cumulative contraction, only increasing $|\mathcal{S}|$ to refine local manifold details once the distribution has stabilized. Fig.\ref{fig:method}(B) visually illustrates the resulting sparse incremental timestep updates.
\end{remark}

\subsection{Feynman-Kac Path Guidance}
\label{sec:fk}
An important drawback of any policy optimization strategy is the sample inefficiency due to terminal rewards obtained only at $\mathbf{x}_0$. However, the Feynman-Kac theory~\cite{Moral2004,Kappen2005} provides a way out by considering each path's prospect at any given timestep. While previous works have used discrete-time DDIM with FK guidance for inference-time sample improvement~\cite{Skreta2025FK_Correctors,Singhal2025FK_Steering}, here, we devise an approach for training.

\paragraph{Feynman--Kac (FK) view and intermediate potentials.}
Optimizing under the FK guidance requires defining a measure that evaluates a path's reward satisfaction, termed as the potential. Following FK steering \cite{singhal2025general, Skreta2025FK_Correctors}, we introduce a sequence of nonnegative \emph{potentials} $\{G_t\}_{t=1}^{T}$ that provide intermediate weighting of partial trajectories. Concretely, given a trajectory $\tau$, we associate the unnormalized weight
\begin{equation}
    w(\tau)\;=\;\prod_{t=1}^{T} G_t(\mathbf{x}_t, c),
\end{equation}
and define an FK-reweighted path distribution
\begin{equation}
    \tilde{p}_\theta(\tau)\;\propto\;p_\theta(\tau)\,\prod_{t=1}^{T} G_t(\mathbf{x}_t,c),
\end{equation}
where $p_\theta(\tau)$ is the trajectory distribution induced by $\pi_\theta$ and the diffusion prior on $\mathbf{x}_T$.

Rather than re-evaluating a denoised prediction at every step, we use FK-style potentials to \emph{allocate compute across a branching sampling tree}. Starting from an initial state $\mathbf{x}_T$, we generate a set of partial trajectories by branching the reverse diffusion transitions. We then complete each branch to obtain terminal samples $\mathbf{x}_0$ and evaluate the terminal reward $R(\mathbf{x}_0,c)$. The resulting terminal rewards are used to assign a scalar potential to each branch (and, by aggregation, to its intermediate nodes), which allows us to prune branches that are unlikely to yield high reward. Fig.~\ref{fig:method} clearly illustrates the path termination and branching for improving alignment and diversity respectively. We provide the detailed algorithm in the supplementary material.

\vspace{-0.3em}
\subsection{Rare Events and Rarity}
\label{sec:rare-events}
In order to correctly quantify desired shifts in distribution, we formalize the concept of rarity. Given a condition $c$, let $p_\phi(\mathbf{x}_0\mid c)$ denote the pretrained model's distribution induced by the reverse diffusion sampler. We define the rare-event set as:
\begin{equation}
    \mathcal{A}_c(\eta) \;=\; \{\mathbf{x}_0 : R(\mathbf{x}_0,c) \ge \eta\}.
\end{equation}
An event is rare under $p_\phi$ when $\mathbb{P}_{\mathbf{x}_0\sim p_\theta(\cdot\mid c)}(\mathbf{x}_0\in\mathcal{A}_c(\eta))$ is small \cite{singhal2025general, Moral2004, hairer2014improved}.

We refer to a condition $c$ as a \emph{rare condition} when it triggers rare events under the sampler, \ie, when the probability mass over high-reward outcomes is small.
We define the \emph{rarity} of a condition as the rare-event probability,
\begin{equation}
    \rho(c;\eta) \;=\; \mathbb{P}_{\mathbf{x}_0\sim p_\phi(\cdot\mid c)}\!\left(\mathbf{x}_0\in\mathcal{A}_c(\eta)\right).
\end{equation}
Smaller $\rho(c;\eta)$ indicates higher rarity (harder conditions).
For a new model with distribution $p_{\theta}(\mathbf{x}_0\mid c)$ induced by the model $\pi_\theta$, we define rarity and rare conditions with the same event set $\mathcal{A}_c(\eta)$ and threshold $\eta$; only the underlying sampling distribution changes from $p_\phi$ to $p_{\theta}$.

\section{Experiments}

\subsection{Experimental Setup}
\label{sec:exp_setup}
\paragraph{Image Generation.}
For image generation, we use Stable Diffusion (SD) v1.4~\cite{rombach2022highresolution} and fine-tune only the UNet via Low-Rank Adaptation (LoRA)~\cite{hu2022lora} for parameter-efficient adaptation, following DDPO~\cite{black2023ddpo}. We evaluate prompts from three image-domain categories: compositional generation (animals \& actions), attribute binding (colors \& fruits), and spatial relations similar to B2DiffuRL\cite{hu2025towards}. Representative prompts include ``a whale riding a bike'', ``brown banana'', and ``a cup on tshirt''. We use the DDIM sampler~\cite{ddim} with settings consistent with prior RL fine-tuning work~\cite{hu2025towards,black2023ddpo}; unless otherwise stated, we use $T=20$ denoising steps and classifier-free guidance scale $w=5.0$. The image reward is CLIPScore, \ie, the similarity between text and image embeddings measured by CLIP~\cite{pmlr-v139-radford21a}. The full prompt list is reported in the supplementary material.
We employ an incremental training curriculum with four stages using $\mathcal{T}_{\mathrm{sch}}=\{5,10,15,20\}$ progressively increasing the updated timesteps. We train each stage is for 10 iterations. For FK-steered sampling, we use $k=4$ particles per prompt and resample every 5 denoising steps (i.e., $\mathcal{T}_{\mathrm{res}}$ with period 5), with FK scale $\lambda_{\mathrm{FK}}=2.0$.

\paragraph{Vanishing Point (VP) Correction.}
VP experiments use the same image-generation pipeline and sampler settings as the general synthesis task to ensure controlled comparisons. We evaluate prompts focused on long-range geometric consistency (e.g., ``railway tracks'' or ``tunnel scenes with clear perspective cues''). We define the reward as the algebraic consensus of the straight lines converging to distinct vanishing points. We call this reward $VP\ Geometry$. We detect line segments using Line Segment Detection (LSD)~\cite{liao2025convex} and apply a custom line grouping algorithm. Additional algorithmic details are in the supplementary material.


\paragraph{3D Indoor Scene Synthesis.}
We adopt the continuous-domain variant of MiDiffusion~\cite{hu2026mixed} as the base diffusion model and use the 3D-FRONT dataset~\cite{fu20213d}. Following prior 3D scene-generation work~\cite{hu2026mixed,paschalidou2021atiss,tang2024diffuscene}, each scene is represented as an unordered set of objects parameterized by semantic class, 3D position, size, and rotation. We train for 2 different rewards separately. First, we define a custom spatial reward for the prompt ``a scene with a TV stand facing a bed,'' which evaluates both object position and orientation. Second, we use a collision reward\cite{pfaff2025steerablescenegenerationpost} that assigns $+1$ to valid scenes and otherwise applies a penalty equal to the total axis-aligned penetration depth across overlapping ground objects. Additional implementation details are provided in the supplementary material.

\paragraph{Methods Compared.}
Our primary comparison includes the pretrained reference model, DDPO~\cite{black2023ddpo}, its backward-progressive extension B$^2$-DiffuRL~\cite{hu2025towards}, and iADD; all start from the same pretrained model. We additionally evaluate recent alternatives on Template~1: inference-time Annealed Importance Guidance (AIG)~\cite{jena2024elucidatingoptimalrewarddiversitytradeoffs}, DanceGRPO~\cite{xue2025dancegrpounleashinggrpovisual}, and BranchGRPO~\cite{li2025branchgrpostableefficientgrpo}. To test compatibility with group-relative optimization, we also replace iADD's per-sample contrastive signal with GRPO-style groupwise advantage normalization, denoted \emph{iADD+GRPO}.

\vspace{-0.3em}
\subsection{Evaluation Metrics}
\label{sec:eval-protocol-metrics}


\noindent\emph{Alignment--Diversity.}
We quantify alignment using the task-specific reward models defined in Sec.~\ref{sec:exp_setup}. We measure image diversity using the Inception Score (IS), following~\cite{hu2025towards}. We additionally report area-under-curve (AUC) metrics to summarize the trade-off between these objectives across operating points. Higher AUC indicates a better overall balance between alignment and diversity. For Template~1, we also report intra-prompt LPIPS: for each prompt, we generate 24 images and average all pairwise LPIPS distances. This complementary metric directly measures visual variation among samples conditioned on the same prompt.

\paragraph{Rarity (Rare-Event Probability).}
\label{sec:exp_rarity}
To quantify performance on rare or difficult conditions, we first fix the rare-condition set $\mathcal{C}_{\mathrm{rare}}$ from the frozen pretrained model and keep this split identical for all methods. Following Sec.~\ref{sec:rare-events}, rarity for each condition is defined by the model-specific rare-event probability under the same event set $\mathcal{A}_c(\eta)$ and threshold $\eta$.

We only report the aggregated rarity score $\mathrm{Rarity}(m)=\frac{1}{|\mathcal{C}_{\mathrm{rare}}|}\sum_{c\in\mathcal{C}_{\mathrm{rare}}}\hat{\rho}_m(c;\eta)$, where $\hat{\rho}_m(c;\eta)$ is the empirical rare-event probability for method $m$ (estimated from $K$ samples per condition; $K=24$ in image experiments), consistent with Sec.~\ref{sec:rare-events}. Following the Q3 rule, $\eta$ is set to the 75th-percentile reward of the frozen pretrained reference distribution for each task, which yields $\eta=0.34$ for image generation and $\eta=0.64$ for vanishing-point correction. Higher $\mathrm{Rarity}(m)$ indicates higher rare-event success probability on the fixed rare-condition set.

\paragraph{3D-Specific Metrics.}
We report Collision ($\downarrow$), using the same penetration-based formulation as the collision reward defined in Sec.~\ref{sec:exp_setup}, and Class Label KL Divergence (CKL, $\downarrow$). The CKL reference distribution is computed from 3D-FRONT~\cite{fu20213d} test scenes that satisfy the spatial reward constraint. To ensure a fair comparison, we evaluate all methods at a common convergence point where they reach the same mean reward threshold during training.

\paragraph{Evaluation Protocol.} We evaluate models at regular intervals across 45 base prompts, generating 1,080 images per checkpoint using deterministic seeds to ensure robust measurement. To guarantee a fair comparison, we report each method at its optimal operating point—defined as the checkpoint closest to the ideal (1,1) on the Reward-Inception Score Pareto frontier.

\subsection{Qualitative Assessment}

We qualitatively compare our method against SD, DDPO, and B$^2$-DiffuRL across all three tasks. Unless otherwise noted, all visualizations are generated from the best operating point of each model after training.

Figure~\ref{fig:image_qualitative_results} shows results on common and rare prompts. On common prompts, DDPO and B$^2$-DiffuRL improve alignment, but our method produces sharper details and more realistic textures while preserving prompt fidelity. On rare prompts, DDPO and B$^2$-DiffuRL either fail to satisfy the prompt or produce cartoonish outputs, whereas our method maintains both semantic alignment and visual quality. These observations are consistent with the quantitative reward--diversity trends reported in Table. ~\ref{tab:results}. We report more common and rare prompts in supplementary.

\begin{figure}[t]
    \centering
    \includegraphics[width=\linewidth, trim={0 120 0 130}, clip]{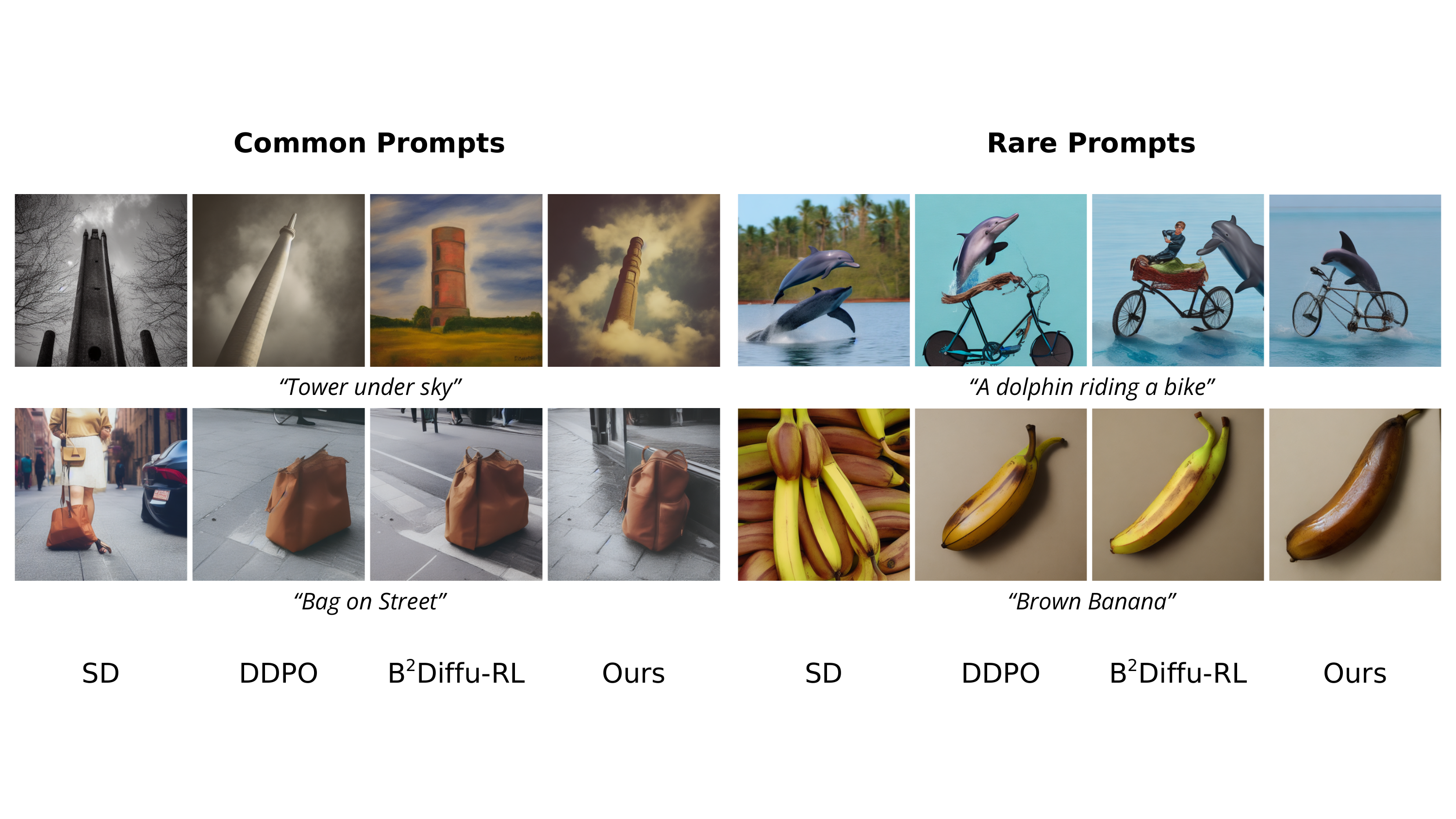}
    \caption{\textbf{Qualitative comparison on common and rare prompts. }Our method better preserves prompt alignment while maintaining realistic visual quality, with larger gains on rare prompts.}
    \label{fig:image_qualitative_results}
\end{figure}

\begin{figure}[h]
    \centering
    \begin{subfigure}[t]{0.46\linewidth}
        \centering
        \includegraphics[width=\linewidth, trim={0 0 0 0}, clip]{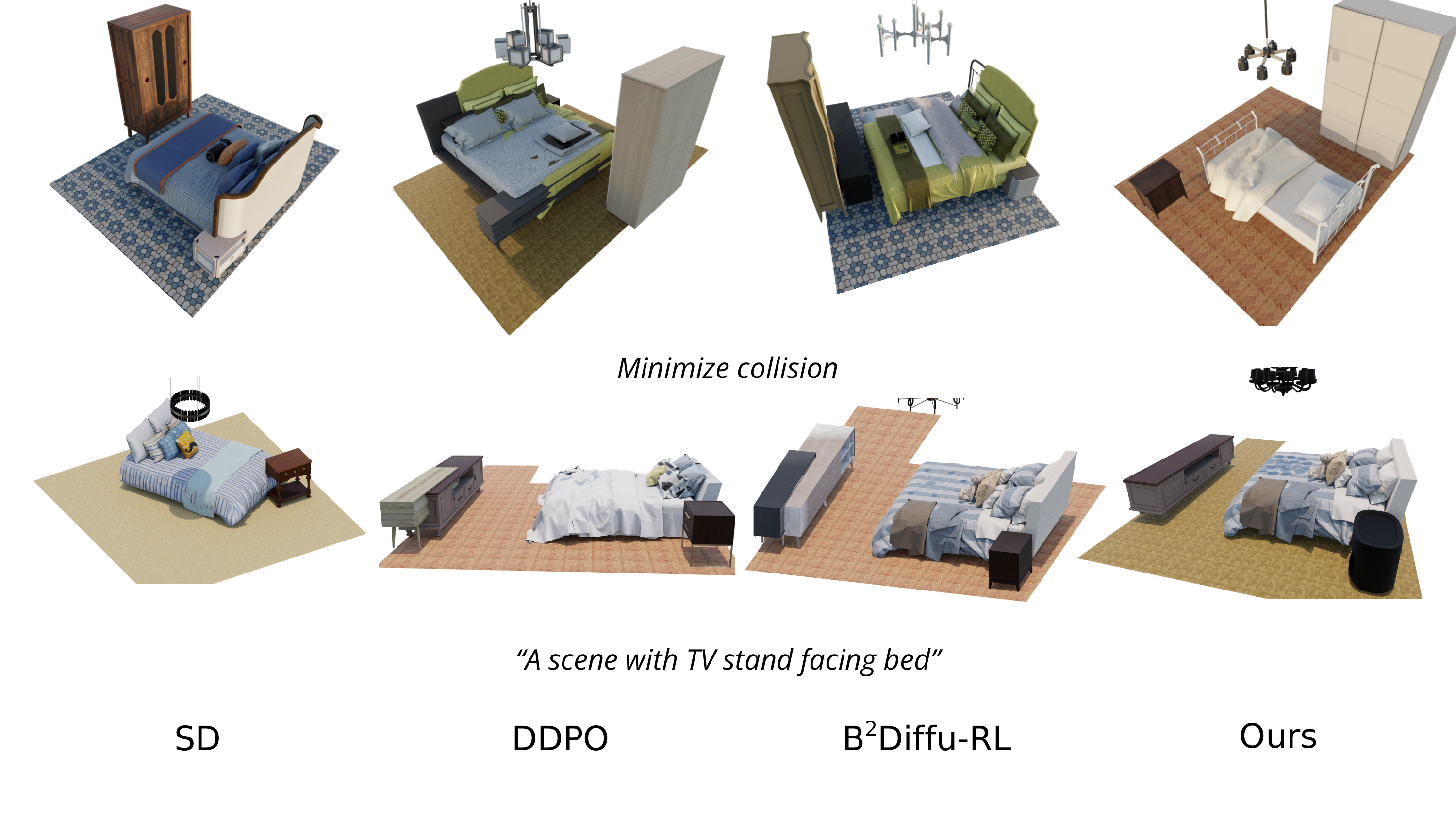}
        \caption{}
        \label{fig:3d_qualitative_results}
    \end{subfigure}\hfill
    \begin{subfigure}[t]{0.53\linewidth}
        \centering
        \includegraphics[width=\linewidth, trim={0 60 0 60}, clip]{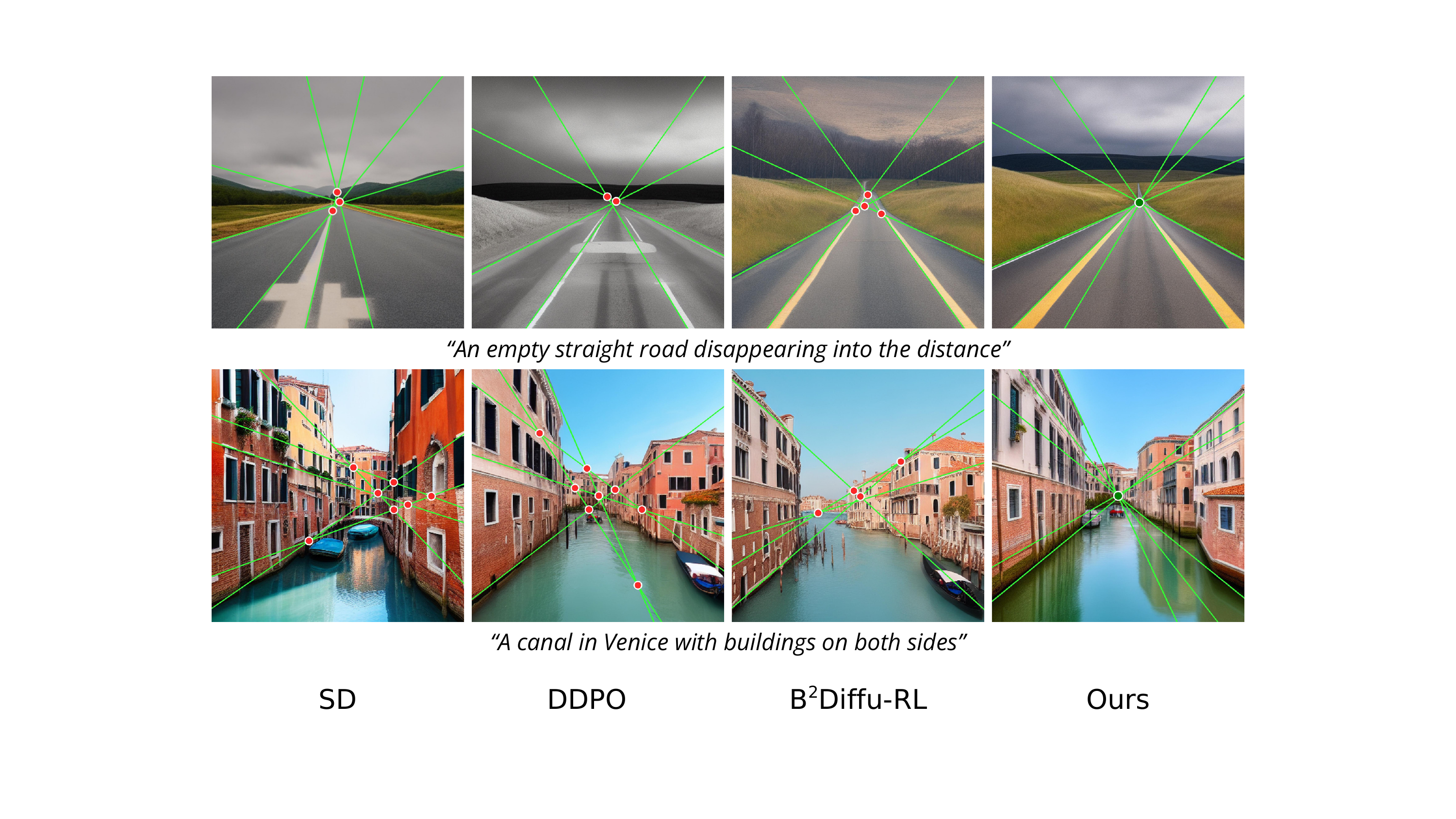}
        \caption{}
        \label{fig:vp_qualitative_results}
    \end{subfigure}
    \vspace{-0.5em}
    \caption{\textbf{Qualitative results.} \textbf{(a) 3D scene synthesis.} Top: collision reward. Bottom: spatial reward (``TV stand facing bed''). DDPO and B$^2$-DiffuRL satisfy constraints but introduce collisions; our method avoids them. \textbf{(b) Vanishing-point correction.} Red: off-vanishing intersections; green: vanishing point. DDPO and B$^2$-DiffuRL shows weaker convergence and desaturation; our method preserves quality while enforcing perspective. (Better viewed zoomed.)}
    \label{fig:qualitative_3d_vp}
    \vspace{-5mm}
\end{figure}

Figure~\ref{fig:vp_qualitative_results} shows that our method yields line structures that converge more consistently to a single vanishing point, while DDPO and B$^2$-DiffuRL often produce weaker or inconsistent convergence. In particular, the first-row DDPO example loses texture fidelity while still failing to enforce clean perspective geometry. Our method preserves texture and local details while satisfying the perspective constraint. This can be attributed to our incremental training schedule with FK-based reweighting and FK branching, which propagates geometric rewards across timesteps without collapsing visual quality.

Figure~\ref{fig:3d_qualitative_results} presents qualitative 3D indoor scene results. For 3D, we evaluate a prompt with explicit geometric constraints ("a scene with TV stand facing bed") (see more prompts in supplementary). Even when optimizing rewards not directly tied to collision minimization, our method better preserves non-colliding layouts than DDPO and B$^2$-DiffuRL, as visible in Figure~\ref{fig:3d_qualitative_results} in the 2nd row and consistent with the observation from Figure~\ref{fig:3d_collision_rate_curve} which reports the corresponding reward--collision trade-off. This shows reward hacking in iADD is lower than other baselines.

\paragraph{Human Evaluation.}
We evaluate perceptual quality and prompt fidelity with a human-preference study over image Templates~1--3. Following the operating-point selection in Sec.~\ref{sec:eval-protocol-metrics}, we sample one image per method for each of 30 prompts and ask 40 independent raters to choose the best output in a blind comparison among the four methods (ours, B$^2$-DiffuRL, DDPO, and the pretrained baseline), yielding 4,800 total judgments. As shown in Figure~\ref{fig:human-eval}, our method is preferred more frequently than DDPO and B$^2$-DiffuRL.

\subsection{Quantitative Results}

\begin{table}[t]
\centering
\caption{\textbf{Performance comparison across metrics and modalities.} Evaluated at their best operating points, our method (iADD) consistently achieves the best alignment-diversity trade-off across all three tasks compared to DDPO and $B^2$-DiffuRL. Notably, iADD maintains higher diversity (better IS and lower collision rates) at comparable or higher reward levels, while also demonstrating state-of-the-art robustness on rare prompts (Rarity). Arrows indicate the desired direction.}
\label{tab:results}
\resizebox{\textwidth}{!}{%
\begin{tabular}{l|cccc|cc|ccc}
\toprule
\textbf{Method}
  & \multicolumn{4}{c|}{\textbf{Image Generation}}
  & \multicolumn{2}{c|}{\textbf{3D Indoor Scene}}
  & \multicolumn{3}{c}{\textbf{Vanishing Point Correction}} \\
\cmidrule(lr){2-5}\cmidrule(lr){6-7}\cmidrule(l){8-10}
& \textbf{Reward$\uparrow$} & \textbf{IS$\uparrow$} & \textbf{Rarity$\uparrow$} & \textbf{AUC$\uparrow$}
&  \textbf{Collision$\downarrow$} & \textbf{CKL$\downarrow$}
& \textbf{Reward$\uparrow$} & \textbf{IS$\uparrow$} & \textbf{Rarity$\uparrow$} \\
\midrule
Pretrained (SD)     & 0.3272 & 1.3514 & - & - & 53.87 & 30.48\% & 0.6105 & 1.2735 & - \\
\midrule
DDPO                & 0.3417 & 1.2786 & 24.81\% & 53.8\% & 63.39\% & 7.726\% & 0.6815 & 1.2600 & 21.33\% \\
B$^2$-DiffuRL              & 0.3452 & 1.2728 & 24.22\% & 58.7\% & 64.06\% & 6.278\% & 0.6486 & 1.2550 & 31.11\% \\
\midrule
\textbf{Ours} & \textbf{0.3553} & \textbf{1.3241} & \textbf{66.85\%} & \textbf{77.7\%} & \textbf{59.20\%} &  \textbf{6.13\%} & \textbf{0.7763} & \textbf{1.2881} & \textbf{95.11\%} \\
\bottomrule
\end{tabular}
}
\end{table}

\begin{table}[t]
\centering
\caption{\textbf{Comparison with recent guidance and GRPO methods on Template~1 (animal activities).} Results use the same prompt set, with trainable methods evaluated at matched checkpoints. iADD provides the strongest standalone diversity and reward--diversity AUC, while iADD+GRPO obtains the highest alignment, rarity, and intra-prompt LPIPS.}
\label{tab:grpo_comparison}
\small
\begin{tabular*}{\textwidth}{@{\extracolsep{\fill}}lccccc@{}}
\toprule
\textbf{Method} & \textbf{CLIP$\uparrow$} & \textbf{IS$\uparrow$} & \textbf{Rarity$\uparrow$} & \textbf{AUC$\uparrow$} & \textbf{LPIPS$\uparrow$} \\
\midrule
AIG (inference) & 0.3526 & 1.3147 & 17.50\% & -- & 0.7531 \\
\midrule
iADD & 0.3830 & \textbf{1.3035} & 78.33\% & \textbf{83.7\%} & 0.7693 \\
DanceGRPO & 0.3886 & 1.2750 & 61.67\% & 39.3\% & 0.7624 \\
BranchGRPO & 0.3854 & 1.2917 & 56.67\% & 43.73\% & 0.7447 \\
iADD+GRPO & \textbf{0.4126} & 1.2942 & \textbf{88.33\%} & 73.3\% & \textbf{0.7778} \\
\bottomrule
\end{tabular*}
\end{table}


\paragraph{Alignment--Diversity Trade-off.}
Figure~\ref{fig:diversity_vs_reward} plots Inception Score (IS) versus CLIPScore  across training for each method. As expected, pushing alignment tends to reduce diversity/quality; however, DDPO and B$^2$-DiffuRL show substantially larger IS degradation at comparable reward levels. In contrast, our method maintains a better reward--IS tradeoff. This is metrically reported in Table~\ref{tab:results}.

\paragraph{Recent guidance and GRPO baselines.}
Table~\ref{tab:grpo_comparison} strengthens this conclusion against more recent alternatives. DanceGRPO and BranchGRPO reach marginally higher CLIPScore than standalone iADD, but incur lower IS, rarity, LPIPS, and substantially lower reward--diversity AUC. AIG preserves high IS at inference time but has lower alignment and rare-event success. Importantly, combining our FK-guided incremental curriculum with GRPO raises CLIPScore from 0.3886 for DanceGRPO to 0.4126 and rarity from 61.67\% to 88.33\%, while also improving LPIPS. Thus, iADD's trajectory guidance and sparse curriculum are complementary to the policy optimizer rather than tied to a specific DDPO objective.

\paragraph{Rarity.}
We report rarity in Table~\ref{tab:results} for both image generation and vanishing-point correction. iADD achieves the strongest rarity performance among all methods. Figure~\ref{fig:common_rare_barplot} shows significant improvement of iADD on rare prompts.

\paragraph{3D Indoor Scene Synthesis.}
Table~\ref{tab:results} shows that, under the same reward target ("a scene with TV stand facing bed") and stopping criterion, DDPO and B$^2$-DiffuRL produce higher collision rates than our method. Figure~\ref{fig:3d_collision_rate_curve} further illustrates this trend across training: competing methods obtain reward gains at the cost of increased collisions, whereas our method maintains lower collision while preserving competitive reward.

\subsection{Ablation Studies}
We conduct controlled ablations to isolate important components for better alignment--diversity--rarity. We present our findings in Table ~\ref{tab:component_disassembly}.

\paragraph{Timestep Selection Ablation.}
\label{sec:empirical-theory}
We first compare standard DDPO with 20-step updates \textit{Vanilla DDPO} against a sparse-update variant that updates any five timesteps \textit{DDPO (any N)}, $N=5$ matched at comparable reward levels. As shown in Fig.~\ref{fig:ablation_diversity_bar}, \textit{DDPO (any N)} achieves higher diversity, supporting Prop.~\ref{prop:any_vs_all}.

Next, we compare branching with late-$N$ updates against branching with any-$N$ updates. Updating any $N$ timesteps consistently yields higher IS than late-$N$ updates (Fig.~\ref{fig:ablation_diversity_bar}), suggesting that the late-timestep strategy used in B$^2$-DiffuRL~\cite{hu2025towards} is suboptimal, as predicted by Prop.~\ref{prop:control_authority}.

We also measure the quantities appearing in the theory during fine-tuning. Figure~\ref{fig:theory_empirical_validation} shows that the cumulative output perturbation $\|\delta x_0\|$ remains larger for uniformly sampled timesteps than for late-only updates throughout training, directly matching Prop.~\ref{prop:control_authority}. At matched compute (stage 27), sparse updates improve the full-generator log-volume measure by $1605$ nats over dense updates ($-25{,}963$ versus $-27{,}568$), with a $437$ SEM for the paired difference across 24 prompt--seed pairs ($3.7$ standard errors). This provides direct empirical support for the volume-preservation prediction of Prop.~\ref{prop:any_vs_all}, beyond the downstream IS measurements.

\begin{figure}[t]
\centering
\begin{minipage}[c]{0.56\textwidth}
    \centering
    \includegraphics[width=\linewidth,trim={15 0 0 0},clip]{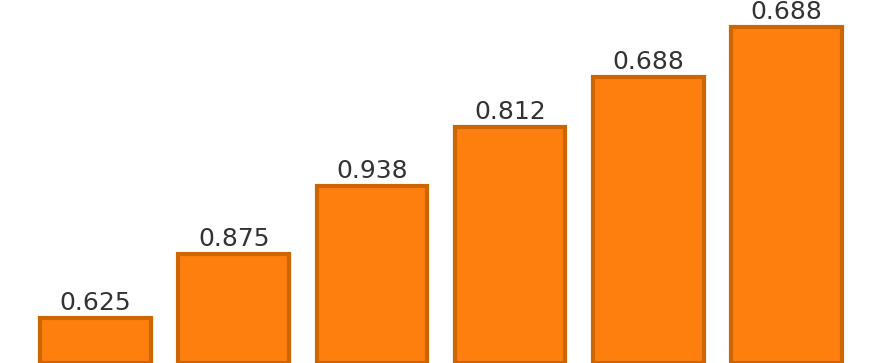}
\end{minipage}\hfill
\begin{minipage}[c]{0.40\textwidth}
    \centering
    \small
    \renewcommand{\arraystretch}{1.15}
    \begin{tabular}{lcc}
    \toprule
    Method & $\log |\det J_{T\to0}|$ & SEM \\
    \midrule
    Dense & $-27{,}568$ & 658 \\
    \textbf{Sparse} & $\mathbf{-25{,}963}$ & 584 \\
    \bottomrule
    \end{tabular}
\end{minipage}
\caption{\textbf{Empirical validation of the timestep analysis.} \textbf{Left:} cumulative $\|\delta x_0\|$ advantage of uniform over late-only updates across training stages (Prop.~\ref{prop:control_authority}). \textbf{Right:} full-Jacobian log-determinant for dense and sparse updates; the less-negative value indicates better volume preservation under sparse optimization (Prop.~\ref{prop:any_vs_all}).}
\label{fig:theory_empirical_validation}
\end{figure}

\paragraph{Ablations on Incremental Strategy.}
Motivated by the gains from any-$N$ updates, we introduce an incremental (curriculum-style) timestep schedule. We compare \textit{FK only} against \textit{FK + Incremental}. Figure~\ref{fig:common_rare_barplot} shows improved alignment on both common and rare prompts, while Fig.~\ref{fig:ablation_diversity_bar} shows improved diversity. We attribute this behavior to stronger reward propagation into earlier denoising steps, which mitigates sparse-reward optimization and encourages broader exploration.

\paragraph{Ablations on FK and FK Branching.}
We evaluate the effect of Feynman--Kac reweighting (FK) and FK Branching on alignment and diversity. Figure~\ref{fig:common_rare_barplot} shows that both FK-based variants improve alignment on common and rare prompts compared with DDPO, when evaluated at a matched diversity target (IS $\approx 1.28$). The CLIPScore gap between \textit{FK only} and \textit{FK + Branching} is modest; however, \textit{FK + Branching} consistently attains higher diversity (Fig.~\ref{fig:ablation_diversity_bar}). This is consistent with the intuition that branching replicates high-potential particles and improves mode coverage during training~\cite{hu2025towards}.
Moreover, the CLIPScore increase from $\textit{DDPO+Incremental}$ to $\textit{Ours (FK Only)}$ in Fig.~\ref{fig:common_rare_barplot} further indicates that Feynman--Kac reweighting improves $Rarity$.
The component study also rules out a purely additive explanation: naive combinations such as branching with late-$N$ updates or FK alone improve IS by only about $0.016$, whereas the complete incremental FK-branching recipe produces the substantially larger gain shown in Fig.~\ref{fig:ablation_diversity_bar}.

\begin{figure}[t]
    \centering
    \begin{subfigure}[t]{0.32\linewidth}
        \centering
        \includegraphics[width=\linewidth,trim=8 8 8 8,clip]{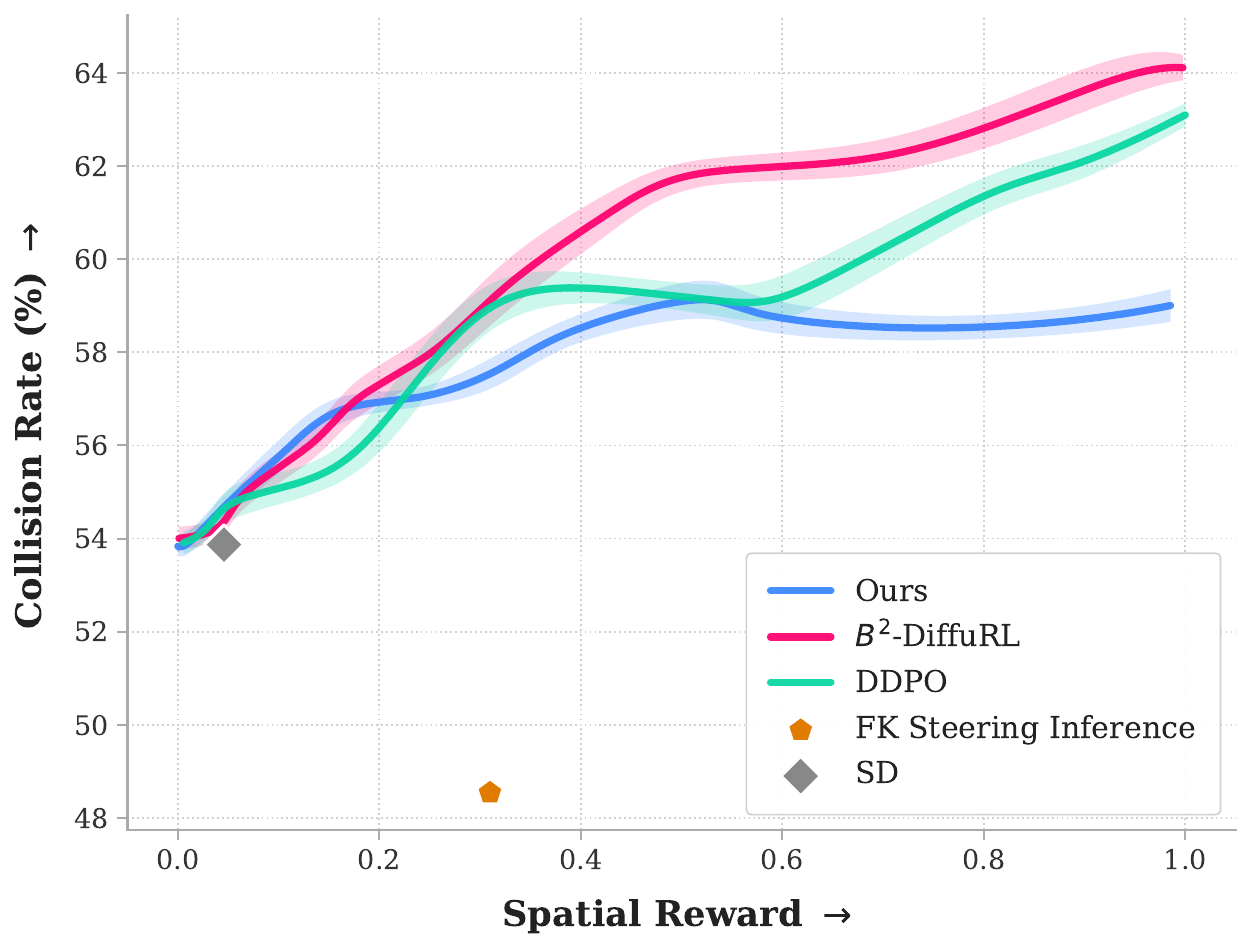}
        \caption{\textbf{Spatial reward vs. collision.} Training with spatial reward only (``TV stand facing bed''). Our method achieves fewer collisions at matched reward levels.}
        \label{fig:3d_collision_rate_curve}
    \end{subfigure}\hfill
    \begin{subfigure}[t]{0.32\linewidth}
        \centering
        \includegraphics[width=\linewidth]{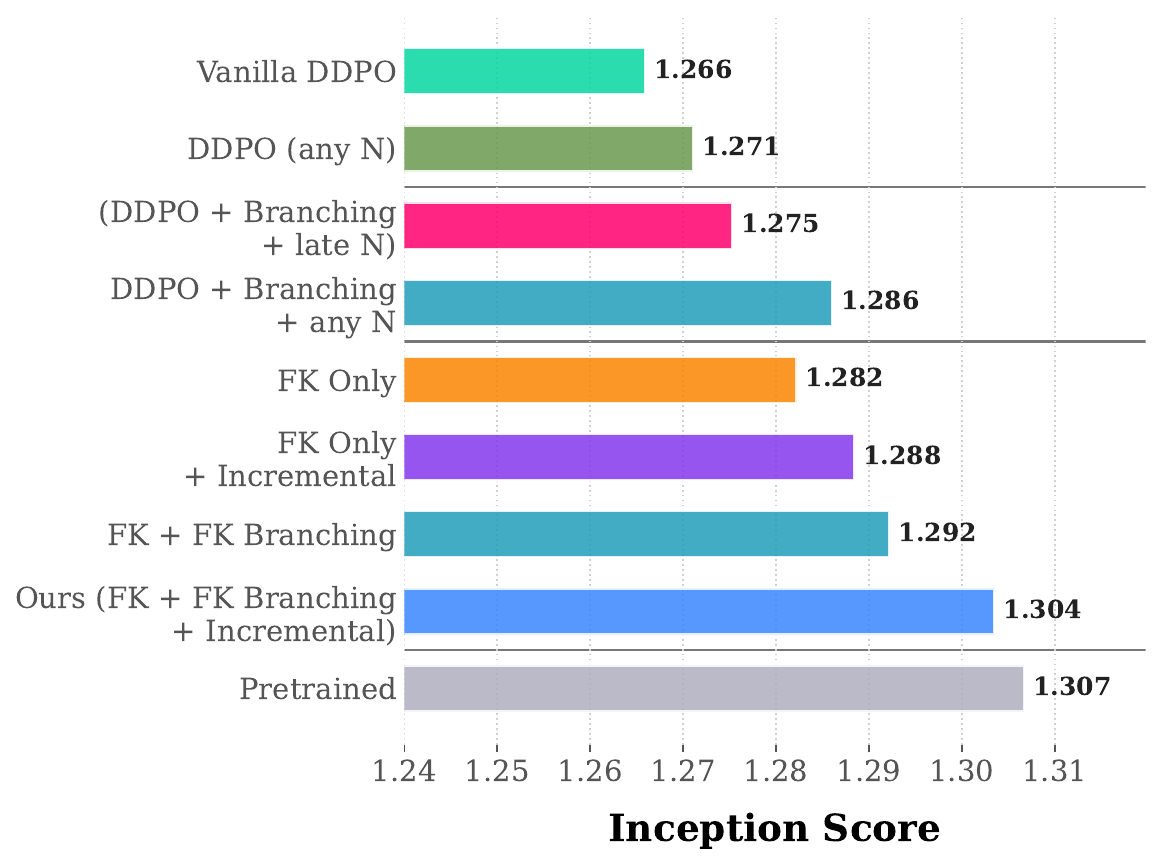}
        \caption{\textbf{Diversity ablation (IS).} Incremental scheduling and FK branching improve diversity over DDPO variants, with our method achieving the best score.}
        \label{fig:ablation_diversity_bar}
    \end{subfigure}\hfill
    \begin{subfigure}[t]{0.32\linewidth}
        \centering
        \includegraphics[width=\linewidth]{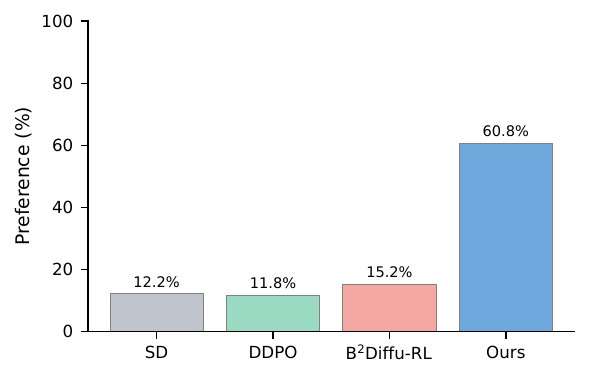}
        \caption{\textbf{Human preferences.} Pairwise ratings from 40 raters on Templates~1--3. Our method is preferred over DDPO and B$^2$-DiffuRL.}
        \label{fig:human-eval}
    \end{subfigure}
    \caption{\textbf{Additional quantitative results.} Our method achieves fewer collisions at matched spatial reward in 3D scene synthesis, improves diversity in the ablation study, and is preferred by human raters over DDPO and B$^2$-DiffuRL.}
    \label{fig:combined_metrics}
\end{figure}
\section{Conclusions}
Our work tackles the joint improvement of alignment, diversity, and rare-event success in diffusion models through reinforcement learning. The theoretical analysis and direct measurements of output sensitivity and Jacobian log-volume show why timestep selection must be treated carefully during reward optimization. This motivates our incremental sparse-update curriculum, while discrete Feynman--Kac pruning and branching improve the quality of the trajectories used for training. Across image generation, vanishing-point correction, and 3D scene synthesis, iADD improves the reward--diversity trade-off and rare-condition performance. Comparisons with AIG, DanceGRPO, and BranchGRPO further show that the proposed trajectory guidance and timestep curriculum remain effective beyond the original DDPO objective and can be composed with GRPO-style optimization.

%
%
\bibliographystyle{splncs04}
\bibliography{main}

\begin{thebibliography}{10}
\providecommand{\url}[1]{\texttt{#1}}
\providecommand{\urlprefix}{URL }
\providecommand{\doi}[1]{https://doi.org/#1}

\bibitem{andrychowicz2018hindsightexperiencereplay}
Andrychowicz, M., Wolski, F., Ray, A., Schneider, J., Fong, R., Welinder, P.,
  McGrew, B., Tobin, J., Abbeel, P., Zaremba, W.: Hindsight experience replay
  (2018), \url{https://arxiv.org/abs/1707.01495}

\bibitem{barcelo2024avoiding}
Barcel{\'o}, R., Alc{\'a}zar, C., Tobar, F.: Avoiding mode collapse in
  diffusion models fine-tuned with reinforcement learning. arXiv preprint
  arXiv:2410.08315  (2024)

\bibitem{Behjoo2025}
Behjoo, H., Chertkov, M.: Harmonic path integral diffusion. IEEE Access
  \textbf{13},  42196--42213 (2025). \doi{10.1109/ACCESS.2025.3548396}

\bibitem{black2023ddpo}
Black, K., Janner, M., Du, Y., Kostrikov, I., Levine, S.: Training diffusion
  models with reinforcement learning. arXiv preprint arXiv:2305.13301  (2023)

\bibitem{Chen2024forcing}
Chen, B., Marti~Monso, D., Du, Y., Simchowitz, M., Tedrake, R., Sitzmann, V.:
  Next-token prediction meets full-sequence diffusion. In: Advances in Neural
  Information Processing Systems (2024)

\bibitem{Chen2018neural}
Chen, R.T.Q., Rubanova, Y., Bettencourt, J., Duvenaud, D.K.: Neural ordinary
  differential equations. In: Advances in Neural Information Processing Systems
  (2018)

\bibitem{Choi2022}
Choi, J., Lee, J., Shin, C., Kim, S., Kim, H., Yoon, S.: Perception prioritized
  training of diffusion models. In: European Conference on Computer Vision
  (2022)

\bibitem{Christiano2017}
Christiano, P.F., Leike, J., Brown, T.B., Martic, M., Legg, S., Amodei, D.:
  Deep reinforcement learning from human preferences. In: Advances in Neural
  Information Processing Systems (2017)

\bibitem{Clark2023}
Clark, K., Vicol, P., Swersky, K., Fleet, D.J.: Directly fine-tuning diffusion
  models on differentiable rewards. In: International Conference on Learning
  Representations (2024), arXiv:2309.17400

\bibitem{DeBortoli2022}
De~Bortoli, V.: Convergence of denoising diffusion models under the manifold
  hypothesis. arXiv preprint, arXiv:2208.05314  (2022)

\bibitem{Moral2004}
Del~Moral, P.: Feynman-kac formulae. In: Feynman-Kac Formulae: Genealogical and
  Interacting Particle Systems with Applications, pp. 47--93. Springer (2004)

\bibitem{devidze2022explorationguided}
Devidze, R., Kamalaruban, P., Singla, A.: Exploration-guided reward shaping for
  reinforcement learning under sparse rewards. In: Oh, A.H., Agarwal, A.,
  Belgrave, D., Cho, K. (eds.) Advances in Neural Information Processing
  Systems (2022), \url{https://openreview.net/forum?id=W7HvKO1erY}

\bibitem{Dhariwal2021}
Dhariwal, P., Nichol, A.: Diffusion models beat gans on image synthesis. In:
  Advances in Neural Information Processing Systems (2021)

\bibitem{fan2023dpok}
Fan, Y., Watkins, O., Du, Y., Liu, H., Ryu, M., Boutilier, C., Abbeel, P.,
  Ghavamzadeh, M., Lee, K., Lee, K.: Dpok: Reinforcement learning for
  fine-tuning text-to-image diffusion models (2023)

\bibitem{fu20213d}
Fu, H., Cai, B., Gao, L., Zhang, L.X., Wang, J., Li, C., Xun, Z., Sun, C., Jia,
  R., Zhao, B., Zhang, H.: 3d-front: 3d furnished rooms with layouts and
  semantics. In: Proceedings of the IEEE/CVF International Conference on
  Computer Vision. pp. 10933--10942 (2021)

\bibitem{futurefu20203dfuture3dfurnitureshape}
Fu, H., Jia, R., Gao, L., Gong, M., Zhao, B., Maybank, S., Tao, D.: 3d-future:
  3d furniture shape with texture (2020),
  \url{https://arxiv.org/abs/2009.09633}

\bibitem{gokmen2025inferring}
Gokmen, A.B., Chhatkuli, A., Van~Gool, L., Paudel, D.P.: Inferring
  compositional 4d scenes without ever seeing one. arXiv preprint
  arXiv:2512.05272  (2025)

\bibitem{gupta2022unpackingrewardshapingunderstanding}
Gupta, A., Pacchiano, A., Zhai, Y., Kakade, S.M., Levine, S.: Unpacking reward
  shaping: Understanding the benefits of reward engineering on sample
  complexity (2022), \url{https://arxiv.org/abs/2210.09579}

\bibitem{hairer2014improved}
Hairer, M., Weare, J.: Improved diffusion monte carlo. Communications on Pure
  and Applied Mathematics  \textbf{67}(12),  1995--2021 (2014)

\bibitem{hare2019dealingsparserewardsreinforcement}
Hare, J.: Dealing with sparse rewards in reinforcement learning (2019),
  \url{https://arxiv.org/abs/1910.09281}

\bibitem{hendrycks2023gaussianerrorlinearunits}
Hendrycks, D., Gimpel, K.: Gaussian error linear units (gelus) (2023),
  \url{https://arxiv.org/abs/1606.08415}

\bibitem{ho2022imagenvideohighdefinition}
Ho, J., Chan, W., Saharia, C., Whang, J., Gao, R., Gritsenko, A., Kingma, D.P.,
  Poole, B., Norouzi, M., Fleet, D.J., Salimans, T.: Imagen video: High
  definition video generation with diffusion models (2022),
  \url{https://arxiv.org/abs/2210.02303}

\bibitem{ddpm}
Ho, J., Jain, A., Abbeel, P.: Denoising diffusion probabilistic models. In:
  Advances in Neural Information Processing Systems (2020)

\bibitem{ho2022videodiffusionmodels}
Ho, J., Salimans, T., Gritsenko, A., Chan, W., Norouzi, M., Fleet, D.J.: Video
  diffusion models (2022), \url{https://arxiv.org/abs/2204.03458}

\bibitem{hu2022lora}
Hu, E.J., Shen, Y., Wallis, P., Allen-Zhu, Z., Li, Y., Wang, S., Wang, L.,
  Chen, W.: Lora: Low-rank adaptation of large language models. In:
  International Conference on Learning Representations (2022)

\bibitem{hu2026mixed}
Hu, S., Arroyo, D.M., Debats, S., Manhardt, F., Carlone, L., Tombari, F.: Mixed
  diffusion for 3d indoor scene synthesis. In: Proceedings of the IEEE/CVF
  Winter Conference on Applications of Computer Vision. pp. 1262--1272 (2026)

\bibitem{hu2020learningutilizeshapingrewards}
Hu, Y., Wang, W., Jia, H., Wang, Y., Chen, Y., Hao, J., Wu, F., Fan, C.:
  Learning to utilize shaping rewards: A new approach of reward shaping (2020),
  \url{https://arxiv.org/abs/2011.02669}

\bibitem{hu2025towards}
Hu, Z., Zhang, F., Chen, L., Kuang, K., Li, J., Gao, K., Xiao, J., Wang, X.,
  Zhu, W.: Towards better alignment: Training diffusion models with
  reinforcement learning against sparse rewards. In: Proceedings of the
  Computer Vision and Pattern Recognition Conference. pp. 23604--23614 (2025)

\bibitem{Hua2025neurips}
Hua, C., Gu, J., Tang, Y.: Continuous q-score matching: Diffusion guided
  reinforcement learning for continuous-time control. In: Advances in Neural
  Information Processing Systems (2025)

\bibitem{jasra2009sequential}
Jasra, A., Doucet, A.: Sequential monte carlo methods for diffusion processes.
  Proceedings of the Royal Society A: Mathematical, Physical and Engineering
  Sciences  \textbf{465}(2112),  3709--3727 (2009)

\bibitem{jena2024elucidatingoptimalrewarddiversitytradeoffs}
Jena, R., Taghibakhshi, A., Jain, S., Shen, G., Tajbakhsh, N., Vahdat, A.:
  Elucidating optimal reward-diversity tradeoffs in text-to-image diffusion
  models (2024)

\bibitem{kang2025rethinkingdpo}
Kang, J., Lim, S., Baek, K., Shim, H.: Rethinking direct preference
  optimization in diffusion models. arXiv preprint arXiv:2505.18736  (2025)

\bibitem{Kappen2005}
Kappen, H.J.: Path integrals and symmetry breaking for optimal control theory.
  Journal of Statistical Mechanics: Theory and Experiment  \textbf{2005}(11),
  P11011 (2005)

\bibitem{li2024diffusionkto}
Li, S., Kallidromitis, K., Gokul, A., Kato, Y., Kozuka, K.: Aligning diffusion
  models by optimizing human utility. In: Advances in Neural Information
  Processing Systems (2024)

\bibitem{Li2023Generalization}
Li, X., Sun, Z., Zhang, Y., Xie, J., Liu, Y.: On the generalization of
  diffusion models. In: International Conference on Learning Representations
  (2023)

\bibitem{li2025branchgrpostableefficientgrpo}
Li, Y., Wang, Y., Zhu, Y., Zhao, Z., Lu, M., She, Q., Zhang, S.: Branchgrpo:
  Stable and efficient grpo with structured branching in diffusion models
  (2025)

\bibitem{liao2025convex}
Liao, B., Zhao, Z., Li, H., Zhou, Y., Zeng, Y., Li, H., Liu, P.: Convex
  relaxation for robust vanishing point estimation in manhattan world. In:
  Proceedings of the Computer Vision and Pattern Recognition Conference. pp.
  15823--15832 (2025)

\bibitem{flow}
Lipman, Y., Chen, R.T.Q., Ben-Hamu, H., Nickel, M., Le, M.: Flow matching for
  generative modeling. In: International Conference on Learning Representations
  (2023)

\bibitem{loshchilov2019decoupledweightdecayregularization}
Loshchilov, I., Hutter, F.: Decoupled weight decay regularization (2019),
  \url{https://arxiv.org/abs/1711.05101}

\bibitem{memarian2021selfsupervisedonlinerewardshaping}
Memarian, F., Goo, W., Lioutikov, R., Niekum, S., Topcu, U.: Self-supervised
  online reward shaping in sparse-reward environments (2021),
  \url{https://arxiv.org/abs/2103.04529}

\bibitem{meng2021sdedit}
Meng, C., He, Y., Song, Y., Song, J., Wu, J., Zhu, J.Y., Ermon, S.: Sdedit:
  Guided image synthesis and editing with stochastic differential equations.
  arXiv preprint arXiv:2108.01073  (2021)

\bibitem{nachum2018dataefficienthierarchicalreinforcementlearning}
Nachum, O., Gu, S., Lee, H., Levine, S.: Data-efficient hierarchical
  reinforcement learning (2018), \url{https://arxiv.org/abs/1805.08296}

\bibitem{Ng1999PolicyIU}
Ng, A., Harada, D., Russell, S.J.: Policy invariance under reward
  transformations: Theory and application to reward shaping. In: International
  Conference on Machine Learning (1999),
  \url{https://api.semanticscholar.org/CorpusID:5730166}

\bibitem{nichol2022glidephotorealisticimagegeneration}
Nichol, A., Dhariwal, P., Ramesh, A., Shyam, P., Mishkin, P., McGrew, B.,
  Sutskever, I., Chen, M.: Glide: Towards photorealistic image generation and
  editing with text-guided diffusion models (2022),
  \url{https://arxiv.org/abs/2112.10741}

\bibitem{instructgpt}
Ouyang, L., Wu, J., Jiang, X., Almeida, D., Wainwright, C., Mishkin, P., Zhang,
  C., Agarwal, S., Slama, K., Ray, A., Schulman, J., Hilton, J., Kelton, F.,
  Miller, L., Simens, M., Askell, A., Welinder, P., Christiano, P., Leike, J.,
  Lowe, R.: Training language models to follow instructions with human
  feedback. Advances in neural information processing systems  \textbf{35},
  27730--27744 (2022)

\bibitem{paschalidou2021atiss}
Paschalidou, D., Kar, A., Shugrina, M., Kreis, K., Geiger, A., Fidler, S.:
  Atiss: Autoregressive transformers for indoor scene synthesis. Advances in
  neural information processing systems  \textbf{34},  12013--12026 (2021)

\bibitem{pfaff2025steerablescenegenerationpost}
Pfaff, N., Dai, H., Zakharov, S., Iwase, S., Tedrake, R.: Steerable scene
  generation with post training and inference-time search (2025),
  \url{https://arxiv.org/abs/2505.04831}

\bibitem{Prabhudesai2024}
Prabhudesai, M., Geng, Z., Pathak, D., Fragkiadaki, K.: Aligning text-to-image
  diffusion models with reward backpropagation. In: European Conference on
  Computer Vision (2024)

\bibitem{alignprop}
Prabhudesai, M., Goyal, A., Pathak, D., Fragkiadaki, K.: Aligning text-to-image
  diffusion models with reward backpropagation  (2023)

\bibitem{qi2017pointnetdeeplearningpoint}
Qi, C.R., Su, H., Mo, K., Guibas, L.J.: Pointnet: Deep learning on point sets
  for 3d classification and segmentation (2017),
  \url{https://arxiv.org/abs/1612.00593}

\bibitem{pmlr-v139-radford21a}
Radford, A., Kim, J.W., Hallacy, C., Ramesh, A., Goh, G., Agarwal, S., Sastry,
  G., Askell, A., Mishkin, P., Clark, J., Krueger, G., Sutskever, I.: Learning
  transferable visual models from natural language supervision. In: Meila, M.,
  Zhang, T. (eds.) Proceedings of the 38th International Conference on Machine
  Learning. Proceedings of Machine Learning Research, vol.~139, pp. 8748--8763.
  PMLR (18--24 Jul 2021),
  \url{https://proceedings.mlr.press/v139/radford21a.html}

\bibitem{rafailov2024directpreferenceoptimizationlanguage}
Rafailov, R., Sharma, A., Mitchell, E., Ermon, S., Manning, C.D., Finn, C.:
  Direct preference optimization: Your language model is secretly a reward
  model (2024), \url{https://arxiv.org/abs/2305.18290}

\bibitem{rombach2022highresolution}
Rombach, R., Blattmann, A., Lorenz, D., Esser, P., Ommer, B.: High-resolution
  image synthesis with latent diffusion models. In: Proceedings of the IEEE/CVF
  Conference on Computer Vision and Pattern Recognition. pp. 10684--10695
  (2022)

\bibitem{saharia2022photorealistic}
Saharia, C., Chan, W., Saxena, S., Li, L., Whang, J., Denton, E., Ghasemipour,
  K., Gontijo~Lopes, R., Ayan, B.K., Salimans, T., Ho, J., Fleet, D.J.,
  Norouzi, M.: Photorealistic text-to-image diffusion models with deep language
  understanding. Advances in neural information processing systems
  \textbf{35},  36479--36494 (2022)

\bibitem{schulman2015trust}
Schulman, J., Levine, S., Abbeel, P., Jordan, M., Moritz, P.: Trust region
  policy optimization. In: International conference on machine learning. pp.
  1889--1897. PMLR (2015)

\bibitem{ppo}
Schulman, J., Wolski, F., Dhariwal, P., Radford, A., Klimov, O.: Proximal
  policy optimization algorithms. arXiv preprint arXiv:1707.06347  (2017)

\bibitem{Silver2025}
Silver, D., Sutton, R.S.: The era of experience. arXiv preprint  (2025),
  manuscript available as PDF (DeepMind).

\bibitem{Singhal2025FK_Steering}
Singhal, R., Horvitz, Z., Teehan, R., Ren, M., Yu, Z., McKeown, K., Ranganath,
  R.: Feynman--kac steering for diffusion models. arXiv preprint  (2025)

\bibitem{singhal2025general}
Singhal, R., Horvitz, Z., Teehan, R., Ren, M., Yu, Z., McKeown, K., Ranganath,
  R.: A general framework for inference-time scaling and steering of diffusion
  models. arXiv preprint arXiv:2501.06848  (2025)

\bibitem{Skreta2025FK_Correctors}
Skreta, M., Akhound-Sadegh, T., Ohanesian, V., Bondesan, R., Aspuru-Guzik, A.,
  Doucet, A., Brekelmans, R., Tong, A., Neklyudov, K.: Feynman--kac correctors:
  Inference-time alignment for diffusion models. arXiv preprint  (2025)

\bibitem{smith2013sequential}
Smith, A.: Sequential Monte Carlo methods in practice. Springer Science \&
  Business Media (2013)

\bibitem{sohl2015deep}
Sohl-Dickstein, J., Weiss, E., Maheswaranathan, N., Ganguli, S.: Deep
  unsupervised learning using nonequilibrium thermodynamics. In: International
  conference on machine learning. pp. 2256--2265. pmlr (2015)

\bibitem{ddim}
Song, J., Meng, C., Ermon, S.: Denoising diffusion implicit models. In:
  International Conference on Learning Representations (2021)

\bibitem{song2021scorebasedgenerativemodelingstochastic}
Song, Y., Sohl-Dickstein, J., Kingma, D.P., Kumar, A., Ermon, S., Poole, B.:
  Score-based generative modeling through stochastic differential equations
  (2021), \url{https://arxiv.org/abs/2011.13456}

\bibitem{tang2024diffuscene}
Tang, J., Nie, Y., Markhasin, L., Dai, A., Thies, J., Nie{\ss}ner, M.:
  Diffuscene: Denoising diffusion models for generative indoor scene synthesis.
  In: Proceedings of the IEEE/CVF conference on computer vision and pattern
  recognition. pp. 20507--20518 (2024)

\bibitem{wallace2023diffusiondpo}
Wallace, B., Feng, M., Kandpal, N., Gardner, M., Singh, S.: Diffusion model
  alignment using direct preference optimization. arXiv preprint
  arXiv:2311.12908  (2023)

\bibitem{wei2023legonetlearningregularrearrangements}
Wei, Q.A., Ding, S., Park, J.J., Sajnani, R., Poulenard, A., Sridhar, S.,
  Guibas, L.: Lego-net: Learning regular rearrangements of objects in rooms
  (2023), \url{https://arxiv.org/abs/2301.09629}

\bibitem{Williams1992REINFORCE}
Williams, R.J.: Simple statistical gradient-following algorithms for
  connectionist reinforcement learning. Machine Learning  \textbf{8}(3-4),
  229--256 (1992)

\bibitem{xu2023imagereward}
Xu, J., Liu, X., Wu, Y., Tong, Y., Li, Q., Ding, M., Tang, J., Dong, Y.:
  Imagereward: Learning and evaluating human preferences for text-to-image
  generation. Advances in Neural Information Processing Systems  \textbf{36},
  15903--15935 (2023)

\bibitem{xue2025dancegrpounleashinggrpovisual}
Xue, Z., Wu, J., Gao, Y., Kong, F., Zhu, L., Chen, M., Liu, Z., Liu, W., Guo,
  Q., Huang, W., Luo, P.: Dancegrpo: Unleashing grpo on visual generation
  (2025)

\bibitem{Yan2025entropy}
Yan, R., Cheng, J., Gan, Y., Sun, S., Wu, Y., Yang, Y., Ling, L., Lin, J., Zhu,
  Y., Zhou, J., Zhang, J., Xing, J., Cai, Y., Huang, R.: Entropy-adaptive
  diffusion policy optimization with dynamic step alignment. In: Proceedings of
  the IEEE/CVF International Conference on Computer Vision. pp. 1924--1934
  (2025)

\bibitem{d3po}
Yang, K., Tao, J., Lyu, J., Ge, C., Chen, J., Shen, W., Zhu, X., Li, X.: Using
  human feedback to fine-tune diffusion models without any reward model. In:
  Proceedings of the IEEE/CVF Conference on Computer Vision and Pattern
  Recognition. pp. 8941--8951 (2024)

\bibitem{zhang2024largescalerl}
Zhang, Y., Tzeng, E., Du, Y., Kislyuk, D.: Large-scale reinforcement learning
  for diffusion models. In: European Conference on Computer Vision (2024)

\bibitem{Zhao2025Score_as_asction}
Zhao, H., Chen, H., Zhang, J., Yao, D.D., Tang, W.: Score as action: Continuous
  reward optimization for diffusion models. In: Proceedings of the Computer
  Vision and Pattern Recognition Conference (2025)

\bibitem{Ziegler2019}
Ziegler, D.M., Stiennon, N., Wu, J., Brown, T.B., Radford, A., Amodei, D.,
  Christiano, P., Irving, G.: Fine-tuning language models from human
  preferences (2019), \url{https://arxiv.org/abs/1909.08593}

\end{thebibliography}

\clearpage

\title{Supplementary Material for iADD: Improving Alignment and Diversity in Diffusion Policy Optimization }

\titlerunning{iADD: Improving Alignment and Diversity in Diffusion Policy Optimization}

\author{Ashok Prasad Neupane\inst{1}\textsuperscript{*} \and
Saugat Adhikari\inst{1}\textsuperscript{*} \and
Pramish Paudel\inst{3}\textsuperscript{*} \and
Ajad Chhatkuli\inst{2,3} \and
Danda Pani Paudel\inst{2,3}}

\authorrunning{A.~P.~Neupane et al.}

\institute{Pulchowk Campus, IOE, Tribhuvan University, Lalitpur, Nepal \and
NAAMII, Kathmandu, Nepal \and
INSAIT, Sofia University ``St. Kliment Ohridski'', Bulgaria\\
\textsuperscript{*}Equal contribution}

\maketitle

\renewcommand{\thesection}{S\arabic{section}}
\renewcommand{\thesubsection}{S\arabic{section}.\arabic{subsection}}
\renewcommand{\thefigure}{S\arabic{figure}}
\renewcommand{\thetable}{S\arabic{table}}
\renewcommand{\theequation}{S\arabic{equation}}
\renewcommand{\thealgorithm}{S\arabic{algorithm}}
\renewcommand{\theHsection}{supp.\arabic{section}}
\renewcommand{\theHsubsection}{supp.\arabic{section}.\arabic{subsection}}
\renewcommand{\theHfigure}{supp.\arabic{figure}}
\renewcommand{\theHtable}{supp.\arabic{table}}
\renewcommand{\theHequation}{supp.\arabic{equation}}
\renewcommand{\theHalgorithm}{supp.\arabic{algorithm}}

\paragraph{Overview of Supplementary Material.} This supplement includes theoretical proofs; algorithm and method details; full experimental setup for image, 3D, and vanishing-point settings (including reward-model details); additional ablations and comparisons (including DPOK~\cite{fan2023dpok}), evaluation metrics, and training-time analysis; extended qualitative results; and discussions of limitations, future work, and broader impacts.

\setcounter{equation}{9}

\section{Theoretical Analysis and Proofs}
\label{sec:proofs}

\begin{proposition}[Control Authority]
\label{supp:prop:control_authority}
    Let $R(\mathbf{x}_0, c)$ be a sparse terminal reward and $\delta \mathbf{\epsilon}_t$ be the parameter update induced by the RL objective at timestep $t$. For a fixed budget of $N$ optimized timesteps, a uniform sampling strategy $\mathcal{S}_{uni} \sim \mathcal{U}(1, T)$ yields a strictly greater total perturbation magnitude $||\delta \mathbf{x}_0||$ than a late-stage strategy $\mathcal{S}_{late} \sim \mathcal{U}(1, N)$.
\end{proposition}
\begin{proof}
    
Consider the DDIM $\pi_\theta$ defined by the discrete update rule:
\begin{equation}
\label{eq:ddim_update}
    \mathbf{x}_{t-1} = \alpha_t \mathbf{x}_t + \beta_t \epsilon_\theta(x_t, t)
\end{equation}
where $\alpha_t$ and $\beta_t$ are the fixed scheduling coefficients and $\epsilon_\theta$ is simply a reparametrization of $\pi_\theta$, as defined in Eq.~(3).

Let us write $\delta \epsilon_t$ as the change induced by the reward $R(\mathbf{x}_0, c)$ on the noise prediction $\epsilon_\theta(x_t, t)$. Then using chain rule, the resulting total perturbation on the clean generation, can be expressed as:

\begin{align}
    \delta x_0
    &= \sum_{t=1}^T
    \frac{\partial x_0}{\partial x_{t-1}}
    \frac{\partial x_{t-1}}{\partial \epsilon_t}\,
    \delta \epsilon_t \\    
    \delta x_0
    &= \sum_{t \in S}
    \mathcal{J}_{t-1 \to 0}\,
    \beta_t\,\delta \epsilon_t .
\end{align}
where, $\mathcal{J}_{t-1 \to 0}$ is the total or compositional Jacobian measuring the cumulative impact of all timesteps. Therefore, the total perturbation $||\delta x_0||$ is the sum of local updates $\delta \epsilon_t$ scaled by the `reach' of that timestep, defined by the Jacobian product $\mathcal{J}_{t-1 \to 0} \beta_t$.
Here, the product of the Jacobians $\prod_{k=1}^{t-1} J_k = \mathcal{J}_{t-1 \to 0}$ and $\beta_t$ is the noise coefficient defined in Eq.~\eqref{eq:ddim_update}. As $t \to 0$, the term $\beta_t$ vanishes $(\beta_t \to 0)$, meaning the network's ability to perturb the trajectory is physically restricted by the noise schedule. Simultaneously, for $t \approx 0$, the Jacobians $J_k$ become strictly contractive $(\|J_k\| < 1)$ as the model enters the manifold-refinement regime~\cite{DeBortoli2022}. Furthermore, at early timesteps, $\beta_t$ is at its maximum magnitude, allowing a small $\delta \mathbf{\epsilon}_t$ to translate into a large coordinate shift in the latent space. It follows that,

\[\sum_{t \in \mathcal{S}_{uni}} \|\mathcal{J}_{t-1 \to 0} \, \beta_t \, \delta \mathbf{\epsilon}_t\| \gg \sum_{t \in \mathcal{S}_{late}} \|\mathcal{J}_{t-1 \to 0} \, \beta_t \, \delta \mathbf{\epsilon}_t\|\]

Thus, the uniform timestep updates provides a much greater control authority on the change $\delta x_0$. It therefore follows that in lack of the control authority, $\pi_\theta$ will have to work much harder, leading to a simplified generated structure in order to fulfill the reward.

\end{proof}

\begin{proposition}[Volume Preservation via Temporal Sparsity]
\label{supp:prop:any_vs_all}
 Let $\mathcal{J}_{T \to 0}$ be the total Jacobian of a $T$-step deterministic reverse diffusion process. For a fixed RL optimization budget, a sparse update strategy over a subset $\mathcal{S}$ (where $|\mathcal{S}| = N \ll T$) preserves the probability volume of the prior distribution strictly better than a full update strategy ($\mathcal{S} = \{1, \dots, T\}$), as measured by the log-determinant of the global Jacobian.
\end{proposition}
\begin{proof}
    We consider the DDIM reverse process as a deterministic flow under the Probability Flow ODE. The change in the probability volume (the determinant of the Jacobian) along the trajectory is governed by Liouville’s Theorem, which states that the rate of change of the log-volume is equal to the divergence of the vector field~\cite{Chen2018neural}.
    
    Let $\tilde{\epsilon}_t$ be the guided noise prediction at step $t$. The local log-determinant of the transition $J_t$ over a step size $\Delta t$ is given by:
    \begin{equation}
      \log |\det J_t| \propto \Delta t \cdot \text{div}(\tilde{\epsilon}_t)
    \end{equation}

    In regions of high reward, the RL-tuned score $\epsilon_\theta(x_t, c, t)$ exhibits a strong "sink" behavior (negative divergence) toward local reward maxima. Consequently, $\text{div}(\tilde{\epsilon}_t^\theta) < \text{div}(\tilde{\epsilon}_t^{ref})$.
    The global Jacobian $\mathcal{J}_{T \to 0}$ for a DDIM chain is the composition of discrete step Jacobians. Its log-determinant is the sum of local step log-determinants:
    \begin{equation}
    \log |\det \mathcal{J}_{T \to 0}| = \sum_{t \in \mathcal{S}} \log |\det J_t^\theta| + \sum_{t \notin \mathcal{S}} \log |\det J_t^{ref}|
    \end{equation}

    In the full update case ($\mathcal{S} = \{1, \dots, T\}$), the negative divergence is integrated across the entire trajectory. Since $\det J_t^\theta < \det J_t^{ref}$ for all $t$, the total volume contracts exponentially. On the other hand, for the sparse updates ($N \ll T$), the $T-N$ frozen steps retain the divergence properties of the pre-trained prior. Thus volume preservation is strictly greater in the latter.
\end{proof}

\begin{remark}[Synthesis: Temporal Curriculum of Generative Reach]
\label{supp:rem:incremental}
The relationship between generative reach (Prop.~\ref{supp:prop:control_authority}) and volume preservation (Prop.~\ref{supp:prop:any_vs_all}) implies an optimal curriculum for reward finetuning: an incremental expansion of the active timestep subset $\mathcal{S}$. By initially restricting optimization to a sparse set of high-leverage early timesteps ($t \sim T$), the model achieves global mode discovery with minimal cumulative contraction, only increasing $|\mathcal{S}|$ to refine local manifold details once the distribution has stabilized.
\end{remark}

\begin{proof}
    Let $\delta \mathbf{x}_0^{(N)}$ be the total perturbation of the clean sample for a subset of size $N$. From Prop. 1, we have $\|\delta \mathbf{x}_0\| \propto \sum_{t \in \mathcal{S}} \beta_t \|\mathcal{J}_{t-1 \to 0}\|$, where the terms $\beta_t$ and $\|\mathcal{J}\|$ are maximal for $t \approx T$. Conversely, from Prop. 2, the volume contraction $\mathcal{C}$ scales with the product of the contractive reward-Hessians: $\det(\mathcal{J}_{T \to 0}) \propto \prod_{t \in \mathcal{S}} \det(I + \gamma_t H_t)$. At initialization, the model requires large $\|\delta \mathbf{x}_0\|$ to traverse the landscape toward high-reward basins. An incremental strategy where $N \ll T$ and $t \in \mathcal{S}$ are sampled from the high-noise regime maximizes the displacement $\|\delta \mathbf{x}_0\|$ per unit of contraction $\mathcal{C}$, as it utilizes the maximum "leverage" $\beta_t$ while the $T-N$ frozen steps act as volume-preserving spacers. As $\nabla R \to 0$ near convergence, the magnitude of $H_t$ diminishes, permitting an increase in $N$ for local refinement without triggering the exponential mode collapse that would occur if $N=T$ were used during the high-gradient discovery phase.
\end{proof}

\section{Method Details}
\label{sec:method_details}

This section provides additional details of the proposed method.

\subsection{Full Training Algorithm}

\label{supp:sec:method-alg}
The full training procedure is summarized in Fig.~\ref{fig:algorithm}.

\begin{figure}[!t]
\centering
\resizebox{0.90\linewidth}{!}{%
\begin{tikzpicture}[node distance=0.50cm]

\node (title1) at (0, 0) [font=\bfseries] {Main Training Loop};

\node (input) [process, below=0.3cm of title1] {
    \textbf{Inputs:} \\
    Pretrained diffusion model $\pi_\phi$, \\ total denoising steps $T$, \\
    set of training prompts C
};

\node (epoch) [process, below=of input] {
    \textbf{for} each epoch:
};

\node (setup) [process, below=of epoch] {
    Setup Number of timesteps to update: $N$
};

\node (sample) [process, below=of setup] {
    Sample trajectories rollout with FK-IPS resampling for each prompt
};

\node (evaluate) [process, below=of sample] {
    Evaluate final samples to get reward
};

\node (select) [process, below=of evaluate] {
    Select best and worst samples from each roll out to create training samples
};

\node (update) [process, below=of select] {
    Update model with $N$ timesteps using DDPO
};

\draw [arrow] (input) -- (epoch);
\draw [arrow] (epoch) -- (setup);
\draw [arrow] (setup) -- (sample);
\draw [arrow] (sample) -- (evaluate);
\draw [arrow] (evaluate) -- (select);
\draw [arrow] (select) -- (update);
\draw [arrow] (update.west) -- ++(-0.8,0) |- (epoch.west);

\node (title2) at (8.0, 0) [font=\bfseries] {Detailed Pseudocode};

\node (code1) [process, below=0.3cm of title2, text width=6.4cm, minimum width=6.4cm,  align=left] {
\lstinline[language=Python]|def get_N(epoch):| \\
\lstinline[language=Python]|    T_schedule = [4, 8, 12, 16]| \\
\lstinline[language=Python]|    idx = epoch // Epochs_per_N| \\
\lstinline[language=Python]|    return T_schedule[min(idx, len(T_schedule))]| \\
\lstinline[language=Python]|    # increase training timesteps over time|
};

\node (code2) [process, below=0.5cm of code1, text width=6.4cm, minimum width=6.4cm, align=left] {
\lstinline[language=Python]|def FK_path_guidance_sampling(model, prompts):| \\
\lstinline[language=Python]|    for t in reversed(T):| \\
\lstinline[language=Python]|        x_t, log_p_t, x_0 = ddim_step(model, x_t, t)| \\
\lstinline[language=Python]|        if t == t_branch:| \\
\lstinline[language=Python]|            x_t = expand_to_k_particles(x_t)| \\
\lstinline[language=Python]|            # split each prompt's particle to be| \\
\lstinline[language=Python]|            # k particles| \\
\lstinline[language=Python]|        if t in T_resample:| \\
\lstinline[language=Python]|            x_t, R = FK_resampling(x_0, prompts, R)| \\
\lstinline[language=Python]|    return trajectories, decoded_images|
\\
\lstinline[language=Python]| | \\
\lstinline[language=Python]|def FK_resampling(x_0, prompts, R_prev):| \\
\lstinline[language=Python]|    R_curr = R(x_0, prompts)| \\
\lstinline[language=Python]|    potentials = max(R_curr, R_prev)| \\
\lstinline[language=Python]|    indices = multinomial_resample(potentials)| \\
\lstinline[language=Python]|    # Filter particles: clone high reward,| \\
\lstinline[language=Python]|    # drop low reward| \\
\lstinline[language=Python]|    return indices, R_curr[indices]|

};

\node (code3) [process, below=0.5cm of code2, text width=6.4cm, minimum width=6.4cm, align=left] {
\lstinline[language=Python]|def inner_training_loop(model, data, N):| \\
\lstinline[language=Python]| #  Sparse Incremental Timestep Update | \\
\lstinline[language=Python]|    for epoch in num_inner_epochs:| \\
\lstinline[language=Python]|        T_train = random_sample(range(T), N)| \\
\lstinline[language=Python]|        for t in T_train:| \\
\lstinline[language=Python]|            ratio = exp(log_p_new - log_p_old)| \\
\lstinline[language=Python]|            loss = -mean(ratio * log_p_new| \\
\lstinline[language=Python]|                * advantages)| \\
\lstinline[language=Python]|            optimize(model, loss)|
};

\draw [arrow] (setup.east) -- ++(1.5,0) |- (code1.west);
\draw [arrow] (sample.east) -- (code2.west);
\draw [arrow] (update.east) -- ++(1.5,0) |- (code3.west);

\node[draw, thick, fit=(title1) (input) (epoch) (setup) (sample) (evaluate) (select) (update), inner sep=10pt] {};
\node[draw, thick, fit=(title2) (code1) (code2) (code3), inner sep=10pt] {};

\end{tikzpicture}%
}
\caption{Pseudo code of iADD}
\label{fig:algorithm}
\end{figure}
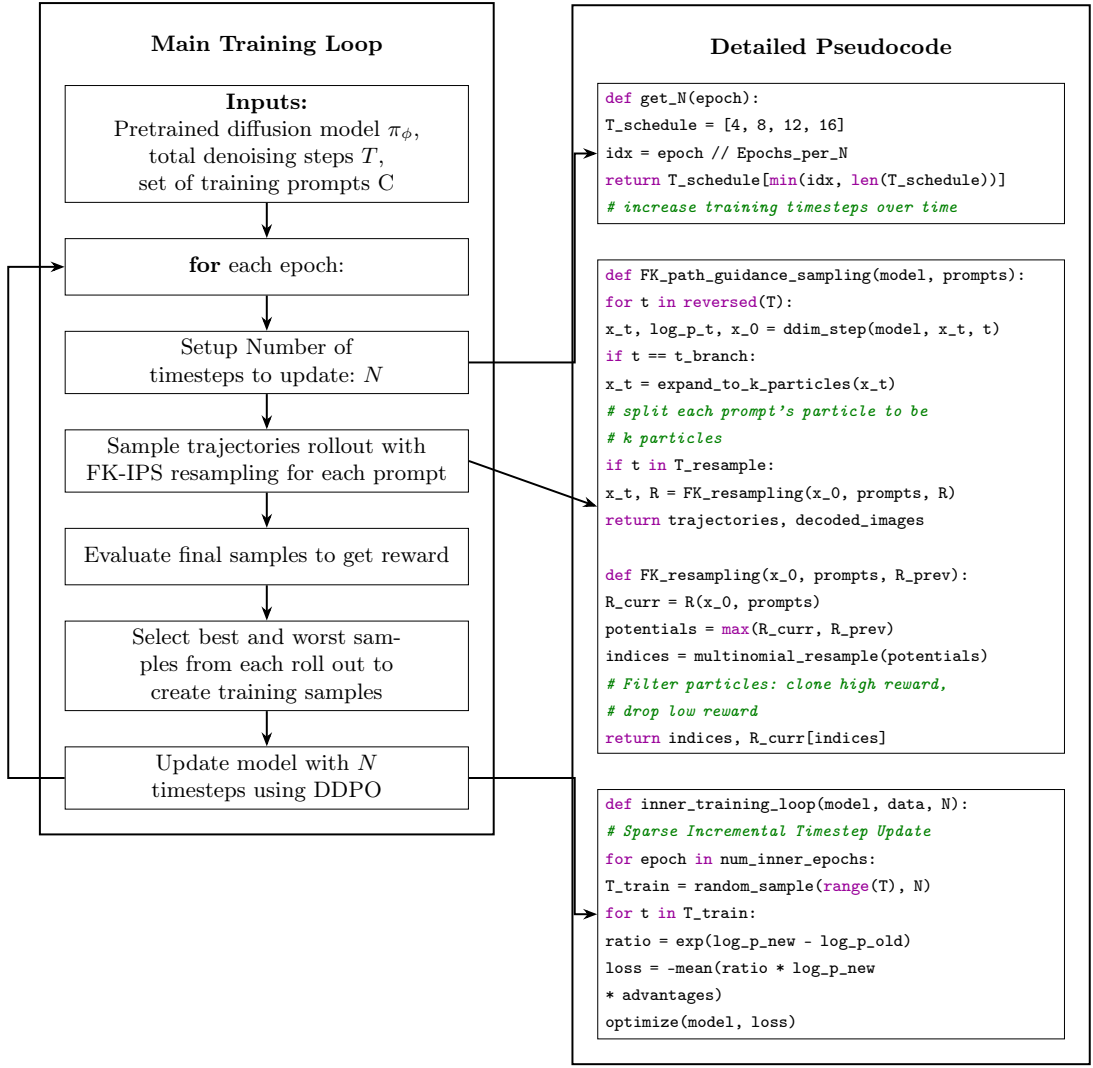

\subsection{Text-to-Image Alignment Reward}

We use CLIP (ViT-H-14)\cite{pmlr-v139-radford21a} to measure alignment between generated images and text prompts.

\begin{algorithm}[H]
\caption{Text-to-Image Alignment Reward}
\small
\begin{algorithmic}[1]
\REQUIRE Generated images $\{I_1, \ldots, I_N\}$, text prompts $\{c_1, \ldots, c_N\}$, CLIP model
\STATE Load CLIP model (ViT-H-14, pretrained on LAION-2B)
\FOR{each batch of images and prompts}
    \STATE Encode images: $\mathbf{f}_I = \text{CLIP}_\text{image}(I)$
    \STATE Encode text: $\mathbf{f}_T = \text{CLIP}_\text{text}(c)$
    \STATE Normalize features: $\mathbf{f}_I \leftarrow \frac{\mathbf{f}_I}{\|\mathbf{f}_I\|}$, $\mathbf{f}_T \leftarrow \frac{\mathbf{f}_T}{\|\mathbf{f}_T\|}$
    \STATE Compute pairwise similarity: $S = \mathbf{f}_I \mathbf{f}_T^\top$
    \STATE Extract diagonal (image-prompt pairs): $r_i = S_{ii}$ \hfill $\triangleright$ Cosine similarity
\ENDFOR
\RETURN Rewards $\mathbf{r} = [r_1, \ldots, r_N]$
\end{algorithmic}
\end{algorithm}
\subsection{Vanishing Point Reward}
The Vanishing Point (VP) reward is designed to encourage the generation of images with strong geometric structures, particularly those following linear perspective. The reward identifies dominant vanishing points in the scene and evaluates the coherence of detected line segments relative to these points.

\begin{algorithm}[H]
\caption{Vanishing Point Reward}
\small
\begin{algorithmic}[1]
\REQUIRE Generated image $I$, number of clusters $K=3$, scaling factor $\gamma=20.0$
\STATE Convert image $I$ to grayscale
\STATE Extract line segments using LSD
\STATE Filter out segments shorter than 20 pixels
\STATE Cluster segments into $K$ groups by orientation $\theta \in [0, \pi)$ using K-Means
\FOR{each cluster $k = 1$ to $K$}
    \STATE Back-project 2D segments to 3D normals: $\mathbf{l} \in \mathbb{R}^3$
    \STATE Generate vanishing point candidates $\mathbf{v} \in \mathbb{S}^2$ from line normal pairs \hfill $\triangleright$ Cross products
    \STATE Select optimal $\mathbf{v}_k$ via RANSAC
    \STATE Compute cluster reward: $R_k = \frac{1}{1 + \gamma \cdot \frac{1}{|L_k|} \sum_{\mathbf{l} \in L_k} |\mathbf{l} \cdot \mathbf{v}_k|}$ \hfill $\triangleright$ Geometric consistency
\ENDFOR
\RETURN $R_{VP} = \frac{1}{K} \sum_{k=1}^K R_k$
\end{algorithmic}
\end{algorithm}

\subsection{3D Indoor Scene Layout Rewards}

We design two rewards to evaluate iADD: one physics-based, encouraging collision-free layouts, and one custom, user-defined spatial constraint, ensuring that the TV stand is positioned in front of the double bed.

\paragraph{Non-penetration Reward}
Minimize object collisions while ignoring ceiling fixtures.

\begin{algorithm}[H]
\caption{Non-penetration Reward}
\small
\begin{algorithmic}[1]
\REQUIRE Object class labels$c_i$, positions $p_i$, sizes $s_i$ for $i=1..N$
\STATE Identify ceiling objects $C$ (e.g., lamps)
\STATE Mask ground objects: $G = \{ i \mid i \notin C \text{ and object } i \text{ not empty} \}$
\FOR{each pair $(i,j)$ in $G$, $i \neq j$}
    \STATE Compute AABB overlap along X, Y, Z axes
    \STATE $\text{penetration}_{ij} = \min(\text{overlap}_x, \text{overlap}_y, \text{overlap}_z)$
\ENDFOR
\STATE Total penetration $P = \sum_{i<j} \text{penetration}_{ij}$
\STATE Reward: $r_\text{non-penetration} = 1$ if $P=0$, else $-P$
\RETURN $r_\text{non-penetration}$
\end{algorithmic}
\end{algorithm}

\paragraph{TV--Bed Placement Reward}
Encourage a TV stand to be in front of a bed at an ideal distance with proper alignment.

\begin{algorithm}[H]
\caption{TV--Bed Placement Reward}
\small
\begin{algorithmic}[1]
\REQUIRE Object positions $p_i$, orientations $o_i$; TV and bed indices; ideal distance $d_\text{ideal}$, variance $\sigma^2$
\STATE Identify TV and bed objects in scene
\IF{bed missing}
    \STATE Penalty: $r \leftarrow r - 5$
\ENDIF
\IF{TV missing}
    \STATE Penalty: $r \leftarrow r - 5$
\ENDIF
\IF{both TV and bed present}
    \STATE Compute distance: $d = \| p_\text{TV} - p_\text{bed} \|$
    \STATE Distance reward: $r_\text{dist} = \exp\left(-\frac{(d - d_\text{ideal})^2}{2\sigma^2}\right)$
    \STATE Compute bed-to-TV direction: $\mathbf{v} = \text{normalize}([p_\text{TV} - p_\text{bed}]_{xz})$ \hfill $\triangleright$ Project to XZ plane
    \STATE Alignment reward: $r_\text{align} = \max(0, o_\text{bed} \cdot \mathbf{v})$ \hfill $\triangleright$ Bed faces TV
    \STATE Total: $r = r_\text{dist} + r_\text{align}$
\ENDIF
\RETURN $r$
\end{algorithmic}
\end{algorithm}

\section{Experimental Setup}
\label{supp:sec:exp_setup}

\subsection{Image Generation}

\paragraph{Model Architecture.}
We use Stable Diffusion v1.4 (SDv1.4)\cite{rombach2022highresolution} as our base model and apply LoRA \cite{hu2022lora} for parameter-efficient fine-tuning. LoRA\cite{hu2022lora} modules are injected into all attention processors across the U-Net architecture, including down-blocks, mid-blocks, and up-blocks.

\paragraph{LoRA Configuration.}
We configure LoRA\cite{hu2022lora} with rank $r=4$, alpha $\alpha=4$, and dropout $p=0.0$. The hidden dimensions for LoRA modules are automatically determined based on the U-Net block structure: for down-blocks and up-blocks, hidden sizes correspond to the block output channels; for mid-blocks, we use the final block output channel dimension. Cross-attention dimension is set based on whether the attention layer is self-attention (attn1) or cross-attention (attn2).

\paragraph{Training Configuration.}
We train on 20 uniformly-spaced timesteps $\mathcal{T}_{\text{train}} = \{1, 51, 101, 151, \ldots, 951\}$ from the total 1000-step diffusion schedule. We use a batch size of 16 for both training and sampling, with 16 batches per epoch. The optimizer is AdamW\cite{loshchilov2019decoupledweightdecayregularization} with learning rate $3 \times 10^{-4}$ and gradient accumulation over 8 steps. No learning rate scheduler is applied. Image-generation experiments are conducted on a single NVIDIA H200 machine using mixed-precision FP16 training, and we use the DDIM \cite{ddim} sampler with $\eta = 1.0$ as the noise scheduler. Number of inner epochs is 2.

\paragraph{Incremental Training Schedule.}
Following our incremental training approach, we progressively expand the set of trained timesteps across 4 stages: $\{5, 10, 15\}$ timesteps, with 10 training epochs per increment.

\paragraph{FK Particle Filtering.}
For particle-based sampling, we use 4 particles per prompt. The branching timestep is set to $t_b = 5$, with resampling applied at timesteps $\{8, 12, 16\}$. We use the maximum potential function with $\lambda = 2.0$ for particle weighting.

\paragraph{Evaluation.}
At inference time, we evaluate on 20 timesteps offset by 1: $\mathcal{T}_{\text{eval}} = \{1, 51, 101, 151, \ldots, 951\}$, following the same uniform spacing and the DDIM sampler\cite{ddim} with $\eta = 1.0$ as the noise scheduler. We use classifier-free guidance scale $w=5.0$ for inference.

\subsection{Vanishing Point Detection}

For vanishing point experiments, we use the same model architecture and hyperparameters as described above for image generation.

\subsection{3D Scene Generation}

\paragraph{Base Model.}
We build upon MiDiffusion\cite{hu2026mixed}, using only the continuous-domain diffusion component. The model is pretrained on bedroom scenes from the 3D-FRONT dataset\cite{fu20213d}. Scene layouts predicted by the model are conditioned on floor plans, and we fetch corresponding CAD models from the 3D-Future\cite{futurefu20203dfuture3dfurnitureshape} asset library for rendering. The model directly predicts the sample ($x_0$) rather than the noise.

\paragraph{Scene Representation.}
Each scene is represented as an unordered set of objects, where each object is parameterized by:
\begin{itemize}
    \item Semantic class: 22-dimensional one-hot vector (including an empty object class)
    \item 3D position: $(x, y, z)$ coordinates
    \item 3D size: $(w, h, d)$ dimensions
    \item 2D rotation: $(\cos\theta, \sin\theta)$ where $\theta$ is the angle with the vertical axis
\end{itemize}

\paragraph{Object and Floor Plan Encoding.}
Following MiDiffusion\cite{hu2026mixed}, we encode object features by processing geometric attributes through an MLP and combining them with learnable class label embeddings. For floor plan conditioning, we sample 256 boundary points from each floor plan image and compute their outward-facing normal vectors. These floor plan features are extracted using a PointNet-based\cite{qi2017pointnetdeeplearningpoint} encoder adapted from LEGO-Net\cite{wei2023legonetlearningregularrearrangements}.

\paragraph{Denoising Network Architecture.}
The denoising network uses a Transformer architecture with 8 layers, embedding dimension of 512, 4 attention heads, feedforward dimension of 2048, and dropout rate of 0.1. We use GELU\cite{hendrycks2023gaussianerrorlinearunits} activation, adaptive layer normalization with absolute timestep encoding, and fully-connected MLP layers. The feature extractor is a PointNet-based architecture with layer dimensions $[4, 64, 64, 512, 64]$. 

\paragraph{LoRA Configuration.}
For 3D scene generation, we apply LoRA\cite{hu2022lora} to all attention projection layers (q\_proj, k\_proj, v\_proj, out\_proj) in the Transformer blocks with rank $r=16$, alpha $\alpha=16$, dropout $p=0.0$, and no bias adaptation. This configuration provides sufficient capacity for fine-tuning while maintaining parameter efficiency.

\paragraph{Training Configuration.}
We use a batch size of 32 for both training and sampling and keep other hyperparameters same as in case of image generation. 3D scene-generation experiments are conducted on a single NVIDIA RTX 2070 Ti machine using mixed-precision FP16 training.

\paragraph{FK Particle Filtering.}
Similar to image generation, we use 4 particles per prompt with the maximum potential function and $\lambda = 2.0$. The branching timestep is set to $t_b = 13$, with resampling applied at timesteps $\{13, 16\}$.

\paragraph{Training and Inference.}
Training and inference timesteps follow the same configuration as image generation: 20 uniformly-spaced timesteps for training and 20 offset timesteps for evaluation, using DDIM sampling with $\eta = 1.0$. The incremental training schedule also mirrors image generation with 4 stages of $\{5, 10, 15, 20\}$ timesteps and 10 epochs per increment.

\subsection{Branch Feynman-Kac Trajectory Rollout}
Branching\cite{hu2025towards} is applied before FK resampling to increase the number of particles from 1 to \texttt{num\_particles} (e.g., 4). In our experiments, we branch at the 5th denoising timestep so that particles share enough early trajectory history and diverge later. This gives multiple trajectories with shared history, allowing us to select paths that produce the best and worst samples.

Inference-time FK steering\cite{singhal2025general} explores different potential designs. In our experiments, max-potential and difference-potential variants achieve similar metrics. We use the following max-potential form:
\[
G_t(x_T, \ldots, x_t, c) = \exp\!\left(\lambda_{\mathrm{FK}} \max_{s=t}^{T} R(x_s, c)\right)
\]

Max potential favors particles with higher rewards. We use this choice because \cite{Singhal2025FK_Steering} reports that max potential better preserves sample diversity.

Similarly, prior FK-steering work studies different values of $\lambda_{\mathrm{FK}}$ (e.g., 2 and 10). Our corresponding comparison is shown in Table~\ref{tab:lambda_fk_ablation}.

\begin{table}[t]
\centering
\caption{Effect of $\lambda_{\mathrm{FK}}$ on Inception Score.}
\label{tab:lambda_fk_ablation}
\begin{tabular}{lc}
\toprule
\textbf{Method} & \textbf{Inception Score} \\
\midrule
Ours ($\lambda_{\mathrm{FK}}=10$) & 1.2830 \\
Ours ($\lambda_{\mathrm{FK}}=2$) & \textbf{1.3035} \\
\bottomrule
\end{tabular}
\end{table}

FK resampling can be applied at different timesteps, but overly frequent resampling may waste computation without significant gains because particles may not have diverged enough. We therefore use resampling at the 8th, 12th, and 16th denoising timesteps. All experiments use 4 particles.

\subsection{Prompt Templates and Dataset}

To evaluate the effectiveness and generalization capabilities of our method, we curated a diverse set of prompt templates spanning several semantic and structural categories. These templates are designed to test the model's ability to handle compositional generalization, out-of-distribution (OOD) attribute binding, spatial reasoning, and geometric regularity.

\paragraph{Prompt Templates}
We utilize four distinct templates\cite{hu2025towards}, each targeting a specific challenge in text-to-image generation:
\begin{itemize}
    \item \textbf{Template 1: Animal-Action Compositionality.} This template pairs various animals (e.g., \textit{cat, shark, kangaroo}) with specific actions (e.g., \textit{riding a bike, playing chess, washing dishes}). The dataset contains 45 prompts.
    \item \textbf{Template 2: Rare Color-Object Binding.} This template focuses on binding colors to fruits and vegetables. We define \textit{Common} prompts as those following natural distributions (e.g., \textit{red apple, yellow banana}) and \textit{Rare} prompts as unnatural or OOD combinations (e.g., \textit{brown banana, green strawberry}). Total 40 prompts are present for this template.
    \item \textbf{Template 3: Spatial Relations.} This template evaluates spatial reasoning using relations such as \textit{on, under, on the left of} and \textit{on the right of}. It contains total of 40 spatial composition prompts to test the model's adherence to relational constraints.
    \item \textbf{Template 4: Geometric Structures.} Designed specifically for evaluating the Vanishing Point reward, this template includes 24 prompts describing scenes with strong linear perspective, such as \textit{railroad tracks stretching to the horizon} or \textit{an empty straight road disappearing into the distance}.
\end{itemize}


\paragraph{Complete Prompt List}
Table~\ref{tab:prompts} provides a comprehensive list of 10 prompts for each template, showcasing the range of scenarios used in our experiments. To maintain clarity, we present the prompts in a structured format that categorizes the training and testing variations.
\begin{table*}[t]
\centering
\caption{Representative prompts from each template category for text-to-image evaluation.}
\label{tab:prompts}
\small
\setlength{\tabcolsep}{4pt}
\begin{tabular}{@{}p{0.22\textwidth}p{0.74\textwidth}@{}}
\toprule
\textbf{Template} & \textbf{Representative Prompts} \\
\midrule
\textbf{T1: Animal-Action} & 
\begin{minipage}[t]{\linewidth}
\begin{enumerate}
    \item a cat washing dishes
    \item a cat playing chess
    \item a dog riding a bike
    \item a dog playing chess
    \item a horse riding a bike
    \item a horse playing chess
    \item a dolphin riding a bike
    \item a gorilla playing chess
    \item a fox riding a bike
    \item a chicken playing chess
\end{enumerate}
\end{minipage} \\
\midrule
\textbf{T2: Rare Binding} & 
\begin{minipage}[t]{\linewidth}
\begin{enumerate}
    \item red apple
    \item yellow banana
    \item orange orange
    \item purple grape
    \item red watermelon
    \item yellow apple
    \item brown banana
    \item green strawberry
    \item black grape
    \item brown kiwi
\end{enumerate}
\end{minipage} \\
\midrule
\textbf{T3: Spatial Relations} & 
\begin{minipage}[t]{\linewidth}
\begin{enumerate}
    \item vase on table
    \item shirt on person
    \item watch on person
    \item jacket on person
    \item motorcycle on road
    \item motorcycle behind person
    \item person behind person
    \item building behind trees
    \item hydrant behind motorcycle
    \item trees behind grass
\end{enumerate}
\end{minipage} \\
\midrule
\textbf{T4: Geometric} & 
\begin{minipage}[t]{\linewidth}
\begin{enumerate}
    \item a straight railroad track leading to a vanishing point at dawn
    \item a long marble hallway in a palace with ornate ceiling
    \item a straight canal through tulip fields in the Netherlands
    \item a bamboo forest path lined with tall bamboo stalks
    \item an infinity pool edge overlooking a mountain valley
    \item railway tracks stretching to the horizon at sunset
    \item an empty straight road disappearing into the distance
    \item a long tunnel with circular arches receding into darkness
    \item a straight highway through a flat desert landscape
    \item a tree-lined avenue with branches forming an arch overhead
\end{enumerate}
\end{minipage} \\
\bottomrule
\end{tabular}
\end{table*}



\section{Evaluation Metrics}
\label{sec:metrics}

This section describes the metrics used to evaluate the performance of our method.

\subsection{Image Generation Metrics}

\paragraph{Inception Score (IS).}
Following B$^2$-DiffuRL, we compute Inception Score using CLIP embeddings instead of the traditional Inception-v3 classifier.

\begin{algorithm}[H]
\caption{CLIP-based Inception Score}
\small
\begin{algorithmic}[1]
\REQUIRE CLIP model, generated images $\{I_1, \ldots, I_N\}$, number of splits $K$
\ENSURE Inception Score: mean $\mu$ and standard deviation $\sigma$
\STATE \textbf{// Step 1: Encode all images}
\FOR{each batch of images}
    \STATE $\mathbf{e}_i = \text{CLIP}_\text{image}(I_i)$ \hfill $\triangleright$ Get image embeddings
    \STATE $p(y|x_i) = \text{softmax}(\mathbf{e}_i)$ \hfill $\triangleright$ Convert to probability distribution
\ENDFOR
\STATE Collect all predictions: $P = [p(y|x_1), \ldots, p(y|x_N)]$
\STATE \textbf{// Step 2: Compute score per split}
\FOR{$k = 1$ to $K$}
    \STATE $P_k \leftarrow P\left[\frac{k-1}{K}N : \frac{k}{K}N\right]$ \hfill $\triangleright$ Split data
    \STATE $p(y) = \frac{1}{|P_k|}\sum_{i} p(y|x_i)$ \hfill $\triangleright$ Marginal distribution
    \STATE \textbf{// Compute KL divergence for each sample}
    \FOR{each $p(y|x_i)$ in $P_k$}
        \STATE $\text{KL}_i = \sum_y p(y|x_i) \log \frac{p(y|x_i)}{p(y)}$ \hfill $\triangleright$ KL$(p(y|x) \| p(y))$
    \ENDFOR
    \STATE $s_k = \exp\left(\frac{1}{|P_k|}\sum_i \text{KL}_i\right)$ \hfill $\triangleright$ Score for split $k$
\ENDFOR
\STATE $\mu = \frac{1}{K}\sum_{k=1}^K s_k$, $\sigma = \sqrt{\frac{1}{K}\sum_{k=1}^K (s_k - \mu)^2}$
\RETURN $(\mu, \sigma)$
\end{algorithmic}
\end{algorithm}

Higher IS indicates both better image quality (low entropy within $p(y|x)$) and greater diversity (high entropy in marginal $p(y)$). For comparison, we sample 1080 samples for every case.

\subsection{3D Indoor Scene Generation Metrics}

\paragraph{Collision Rate.}
We measure the percentage of objects involved in collisions across all generated scenes. Collisions are detected using Axis-Aligned Bounding Box (AABB) overlap tests, where two objects are considered colliding if their bounding boxes intersect. The collision rate is computed as:
\begin{equation}
\text{Collision Rate} = \frac{\text{Number of objects with collisions}}{\text{Total number of objects}} \times 100\%
\end{equation}
Lower collision rates indicate physically more plausible scene layouts. We exclude ceiling-mounted objects (e.g., lamps) from collision detection as they do not interact with ground-level furniture.


\section{Prompt Rarity Analysis}
\label{sec:rarity}

This section provides analysis of prompt rarity used in our experiments.

\subsection{Reward Distribution Across Prompts}
We analyze the performance of the pretrained Stable Diffusion (SD) model across a wide variety of prompts. By sorting prompts using their average pretrained CLIP reward (or Vanishing Point reward for geometric templates), we observe a characteristic reward curve that highlights the model's performance variance. As shown in Fig.~\ref{fig:reward_grid}, the rewards span a significant range, where prompts at the lower end represent concepts the pretrained model finds challenging to represent accurately.


\subsection{Rare vs Common Prompt Definition}
To systematically evaluate model improvements, we define ``Rare'' and ``Common'' prompts based on the pretrained reward distribution. For each template, we sort all prompts by their pretrained reward and identify the bottom 25\% (Q3) as ``Rare'' prompts and the top 25\% (Q1) as ``Common'' prompts. The remaining 50\% of prompts are categorized as intermediate. This definition allows us to measure how fine-tuning methods improve performance in regions where the original model is most deficient.

\subsection{Quartile Analysis}
We further partition the entire evaluation prompt set (aggregated across all templates) into quartiles based on pretrained rewards. As shown in Fig.~\ref{fig:quartile_analysis}, our proposed method (\textit{Ours}) consistently out-performs SD and competitive baselines across all quartiles. Notably, the relative improvement is most pronounced in the third quartile (Rare prompts), where \textit{Ours} achieves a 16.7\% increase over the SD baseline. The gains are sustained at 12.0\% for Intermediate (Q2) and 10.5\% for Common (Q3) prompts, demonstrating that our method effectively shifts the reward distribution upward across the entire spectrum of prompt difficulty.

\begin{figure*}[t]
    \centering
    \includegraphics[width=\textwidth]{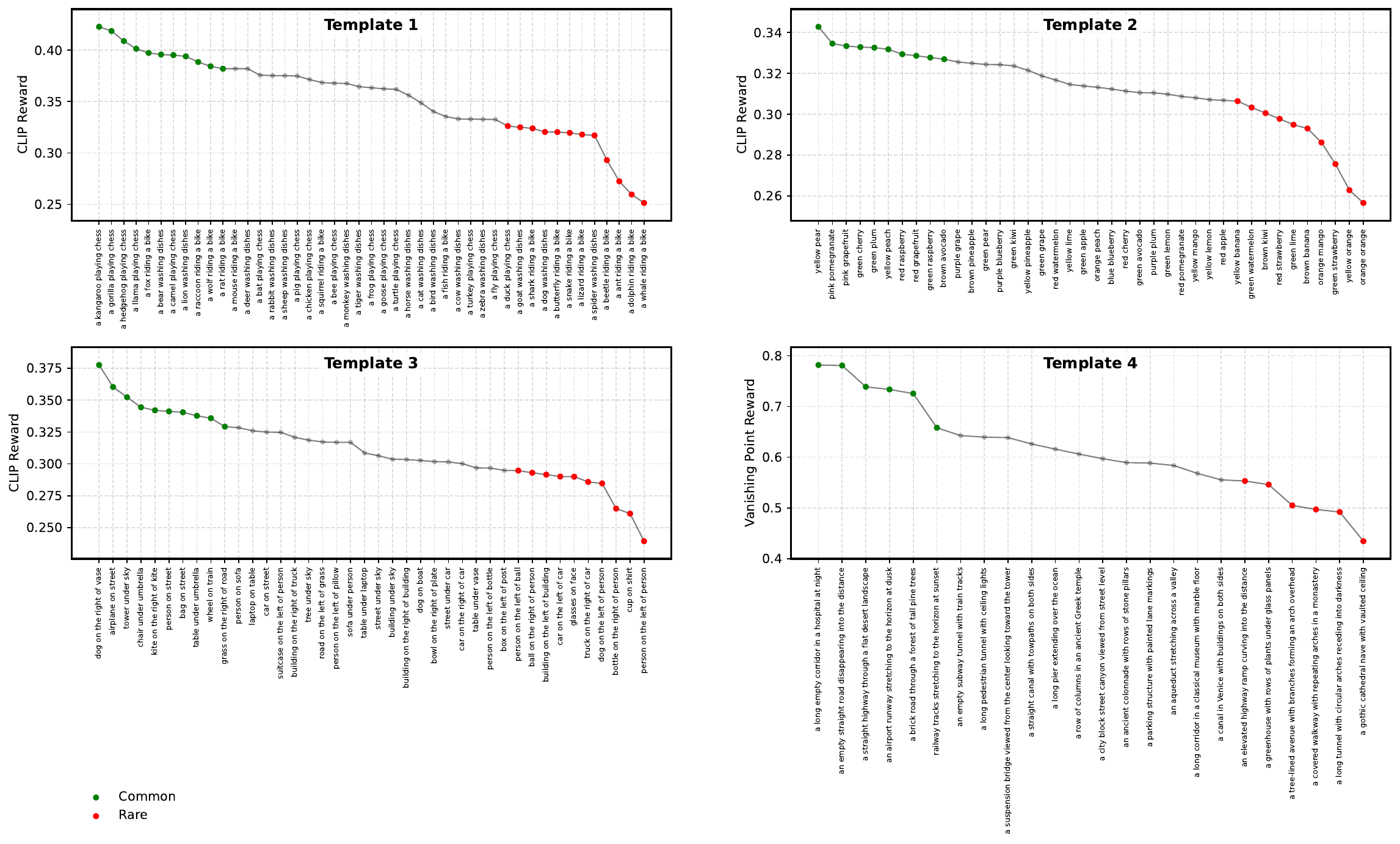}
    \caption{\textbf{Pretrained Reward Distribution across Templates.} We visualize the CLIP reward (T1-T3) and Vanishing Point reward (T4) of the pretrained Stable Diffusion model across sorted evaluation prompts. For each template, we define \textit{Rare} prompts as those in the bottom 25\% (red) and \textit{Common} prompts as those in the top 25\% (green). The characteristic curve highlights significant performance variance in the base model, particularly for challenging prompts in T1-T3.}
    \label{fig:reward_grid}
\end{figure*}

\begin{figure}[t]
    \centering
    \includegraphics[width=\linewidth]{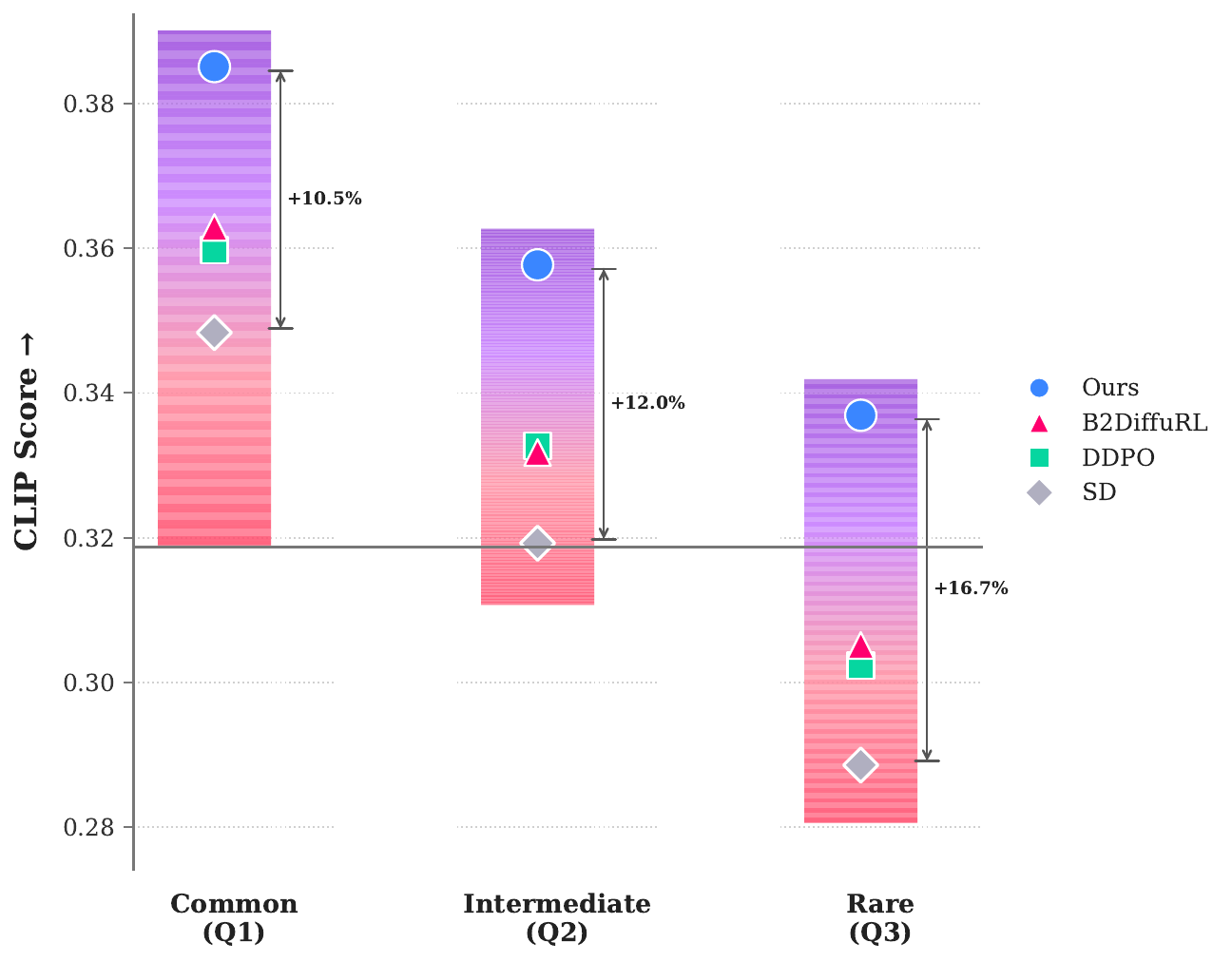}
    \caption{\textbf{Quartile Performance Analysis.} Comparison of model performance across prompt difficulty quartiles (Q1: Common, Q2: Intermediate, Q3: Rare). The horizontal line represents the global pretrained SD baseline. Our method (\textit{Ours}) consistently outperforms SD, DDPO, and B$^2$-DiffuRL, with the largest relative gain observed in the most challenging \textit{Rare} quartile ($+16.7\%$), while maintaining significant leads in intermediate ($+12.0\%$) and common ($+10.5\%$) scenarios.}
    \label{fig:quartile_analysis}
\end{figure}

\section{Additional Experiments}
\label{sec:extra_exp}

This section contains additional experiments that further validate the proposed method.

\subsection{Comparison with DPOK}

DPOK\cite{fan2023dpok} suggests that incorporating a KL-divergence regularization term during optimization can mitigate reward hacking while improving sample diversity. The intuition is that constraining the optimized policy to remain close to the reference model prevents the model from exploiting weaknesses in the reward function.

To evaluate this claim in our setting, we compare our method with DPOK-style KL regularization and other baseline methods. Following the evaluation protocol, we measure diversity while keeping the reward fixed at $0.38$, and measure reward while keeping the diversity fixed at $1.285$. This controlled evaluation isolates the trade-off between reward maximization and diversity.

\begin{table}[t]
\centering
\begin{tabular}{lcccc}
\hline
\textbf{Method} & \textbf{Reward} $\uparrow$ & \textbf{Diversity} $\uparrow$ & \textbf{Rarity} $\uparrow$ & \textbf{AUC} $\uparrow$ \\
\hline
SD\cite{rombach2022highresolution} & 0.3565 & 1.3067 & - & - \\
\hline
DPOK\cite{fan2023dpok} & 0.3582 & 1.2855 & 28.33 & - \\
B$^2$Diff-RL\cite{hu2025towards} + KL regularization & 0.3646 & 1.2784 & 27.5 & -\\
Ours + KL regularization & \textbf{0.3800} & \textbf{1.3002} & \textbf{55.83} & \textbf{-} \\
\hline
\end{tabular}
\caption{Comparison of different methods under KL-divergence regularization. Diversity is measured while fixing the reward to $0.38$, and reward is measured while fixing diversity to $1.285$. While KL regularization helps reduce reward hacking as suggested by DPOK, our method still achieves higher diversity and reward than existing approaches.}
\label{tab:kl_comparison}
\end{table}

While KL-divergence regularization stabilizes optimization and reduces reward exploitation, it introduces an additional constraint that may limit the achievable performance. In our experiments, although the KL-regularized variant improves robustness against reward hacking, it does not match the performance of our method without KL regularization. This suggests that our optimization strategy already provides a better balance between reward maximization and diversity without requiring explicit KL constraints.
\subsection{Generalization Across Prompts VP}


\section{Additional Quantitative Results}
\label{sec:quantitative}

This section provides additional quantitative results supporting the effectiveness of our approach.

\subsection{Template-wise Tables for Image Generation}

We present the detailed quantitative results in Tables~\ref{tab:results_t1}, \ref{tab:results_t2}, and~\ref{tab:results_t3} for each individual template (T1, T2, and T3) used in our text-to-image generation experiments.

\begin{table}[h]
    \centering
    \begin{tabular}{lcccc}
        \toprule
        Method & Alignment  $\uparrow$ & Diversity $\uparrow$ & Rarity $\uparrow$ & AUC $\uparrow$  \\
        \midrule
        SD & - & - & - \\
        DDPO & - & - & - \\
        B$^2$-DiffuRL & - & - & - \\
        \textbf{Ours} & - & - & - \\
        \bottomrule
    \end{tabular}
    \vspace{0.1in}
    \caption{Quantitative results for Template T1.}
    \label{tab:results_t1}
\end{table}

\begin{table}[h]
    \centering
    \begin{tabular}{lcccc}
        \toprule
        Method & Alignment  $\uparrow$ & Diversity $\uparrow$ & Rarity $\uparrow$ & AUC $\uparrow$  \\
        \midrule
        SD & - & - & - \\
        DDPO & - & - & - \\
        B$^2$-DiffuRL & - & - & - \\
        \textbf{Ours} & - & - & - \\
        \bottomrule
    \end{tabular}
    \vspace{0.1in}
    \caption{Quantitative results for Template T2.}
    \label{tab:results_t2}
\end{table}

\begin{table}[h]
    \centering
    \begin{tabular}{lcccc}
        \toprule
        Method & Alignment  $\uparrow$ & Diversity $\uparrow$ & Rarity $\uparrow$ & AUC $\uparrow$  \\
        \midrule
        SD & - & - & - \\
        DDPO & - & - & - \\
        B$^2$-DiffuRL & - & - & - \\
        \textbf{Ours} & - & - & - \\
        \bottomrule
    \end{tabular}
    \vspace{0.1in}
    \caption{Quantitative results for Template T3.}
    \label{tab:results_t3}
\end{table}

\subsection{Vanishing Point Evaluation}

In this section, we evaluate our method on the geometric vanishing-point task. Figure~\ref{fig:vp_reward_is} illustrates the Inception Score (IS) versus reward curve. Unlike standard settings, where alignment improvements often come at the cost of image quality, our method improves both reward and IS simultaneously. One possible reason is that the line-detection-based reward favors sharper, more structured images, which can also increase IS.

However, we argue that the IS might not be a highly reliable metric for evaluating the true quality of images in scenarios with strict geometric constraints, such as vanishing points. To robustly verify that the visual quality of the images generated by our method is actually better than the baselines, we conducted a comprehensive user study. The user study results (shown in Figure~\ref{fig:vp_user_study}) confirm that human evaluators consistently prefer the visual quality and alignment of our method over the alternatives.

\begin{figure}[h]
    \centering
    \includegraphics[width=0.6\linewidth]{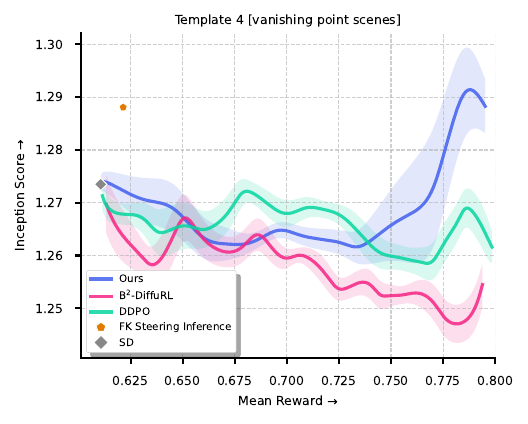}
    \caption{Inception Score vs Mean Reward plot for the vanishing point task. Our method achieves higher alignment while the inception score also increases, demonstrating an improvement in both metrics without a strict tradeoff.}
    \label{fig:vp_reward_is}
\end{figure}

\begin{figure}[h]
    \centering
    \begin{subfigure}{0.48\linewidth}
        \includegraphics[width=\linewidth]{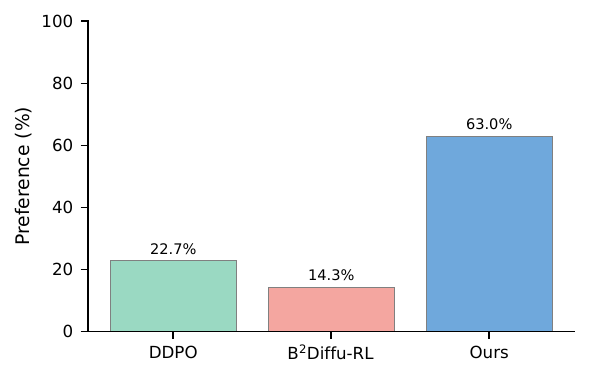}
        \caption{Global}
    \end{subfigure}\hfill
    \begin{subfigure}{0.48\linewidth}
        \includegraphics[width=\linewidth]{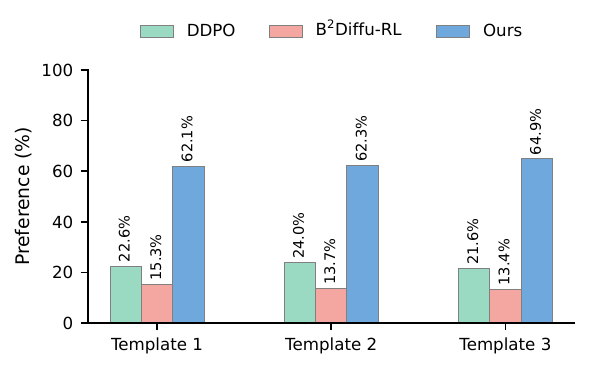}
        \caption{Per Template}
    \end{subfigure}
    \\\vspace{0.6em}
    \begin{subfigure}{0.8\linewidth}
        \includegraphics[width=\linewidth]{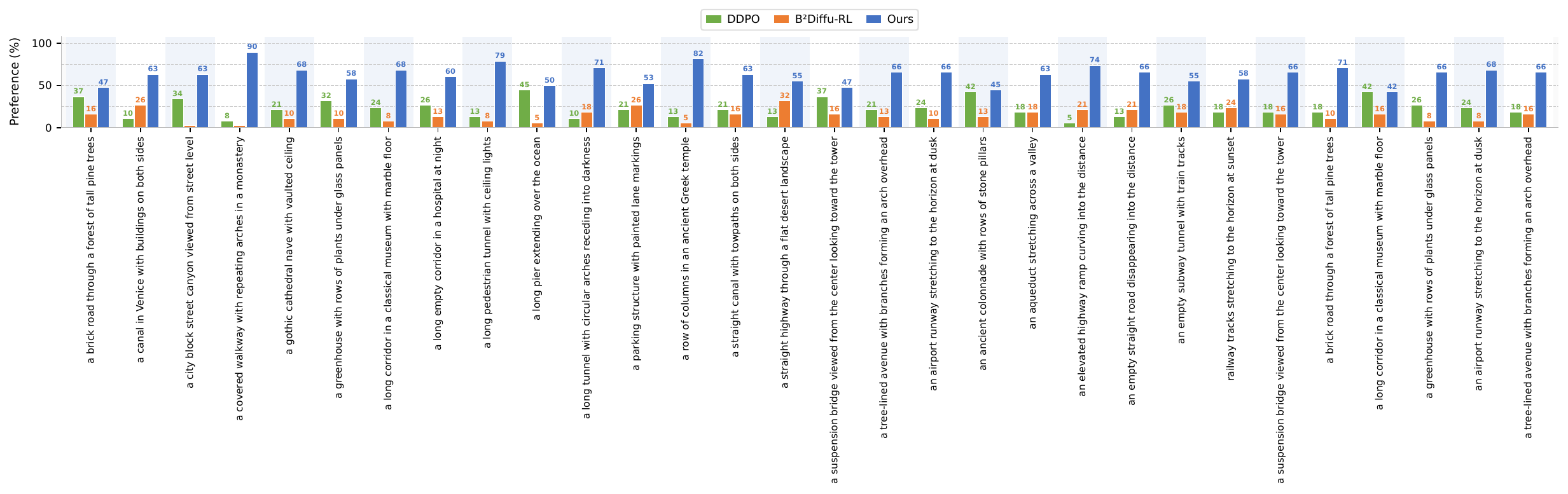}
        \caption{Per Prompt}
    \end{subfigure}
    \caption{User study results for the vanishing point task, showing global preference, per-template preference, and per-prompt preference. Human evaluators consistently preferred our method.}
    \label{fig:vp_user_study}
\end{figure}

\subsection{User Study Results for Image Generation}

In addition to the geometric evaluation, we provide the detailed user study plots for the standard image generation templates. Figure~\ref{fig:img_user_study} visualizes the human preference rates broken down both per template and per prompt.

\begin{figure}[h]
    \centering
    \begin{subfigure}{0.6\linewidth}
        \includegraphics[width=\linewidth]{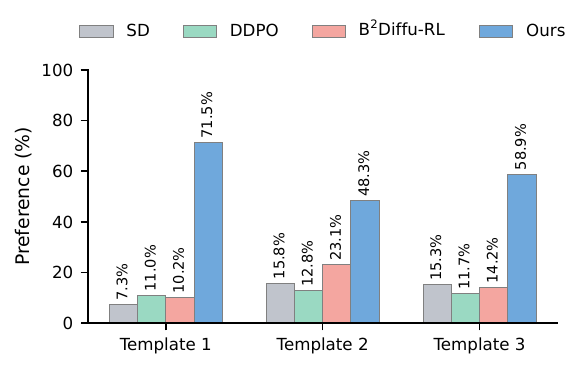}
        \caption{Per Template}
    \end{subfigure}
    \\\vspace{0.6em}
    \begin{subfigure}{0.8\linewidth}
        \includegraphics[width=\linewidth]{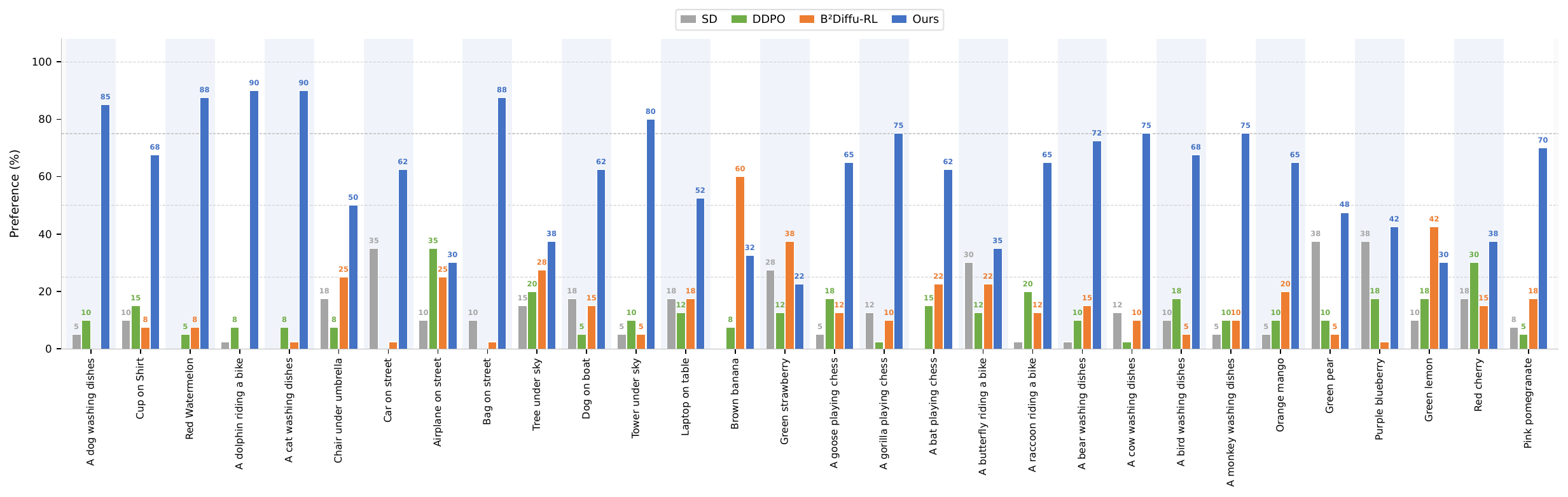}
        \caption{Per Prompt}
    \end{subfigure}
    \caption{User study results for general image synthesis broken down per template and per prompt.}
    \label{fig:img_user_study}
\end{figure}

\subsection{Quartile Reward Analysis}

To provide a deeper understanding of the performance across different difficulty levels, we plot the quartile rewards for each template. Figures~\ref{fig:quartile_t1_t2} and~\ref{fig:quartile_t3_geo} illustrate how our method performs across the 25th, 50th, and 75th percentiles of prompt difficulties for templates T1, T2, T3, and the vanishing point task.

\begin{figure}[h]
    \centering
    \begin{subfigure}{0.48\linewidth}
        \includegraphics[width=\linewidth]{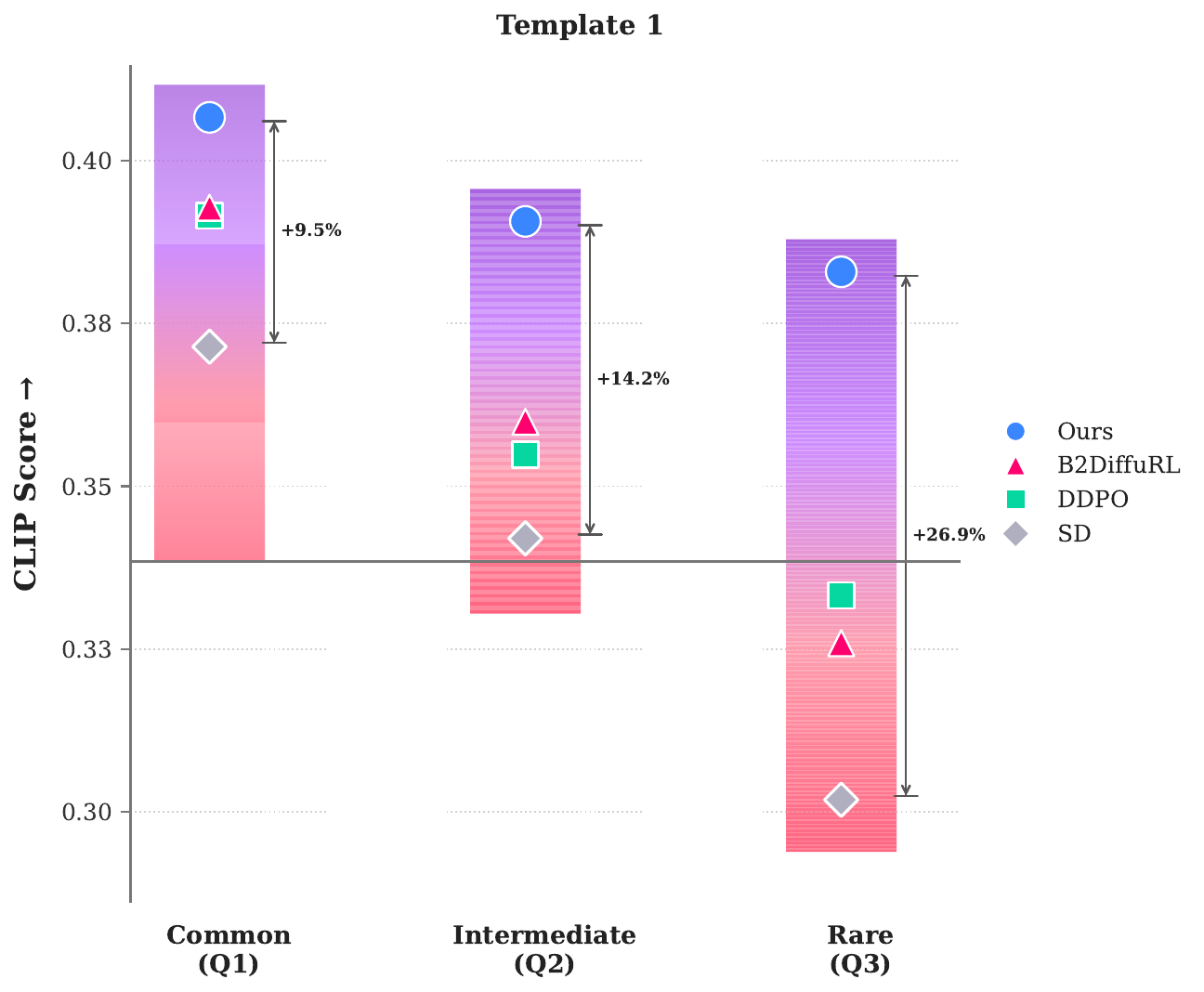}
        \caption{Template T1}
    \end{subfigure}
    \begin{subfigure}{0.48\linewidth}
        \includegraphics[width=\linewidth]{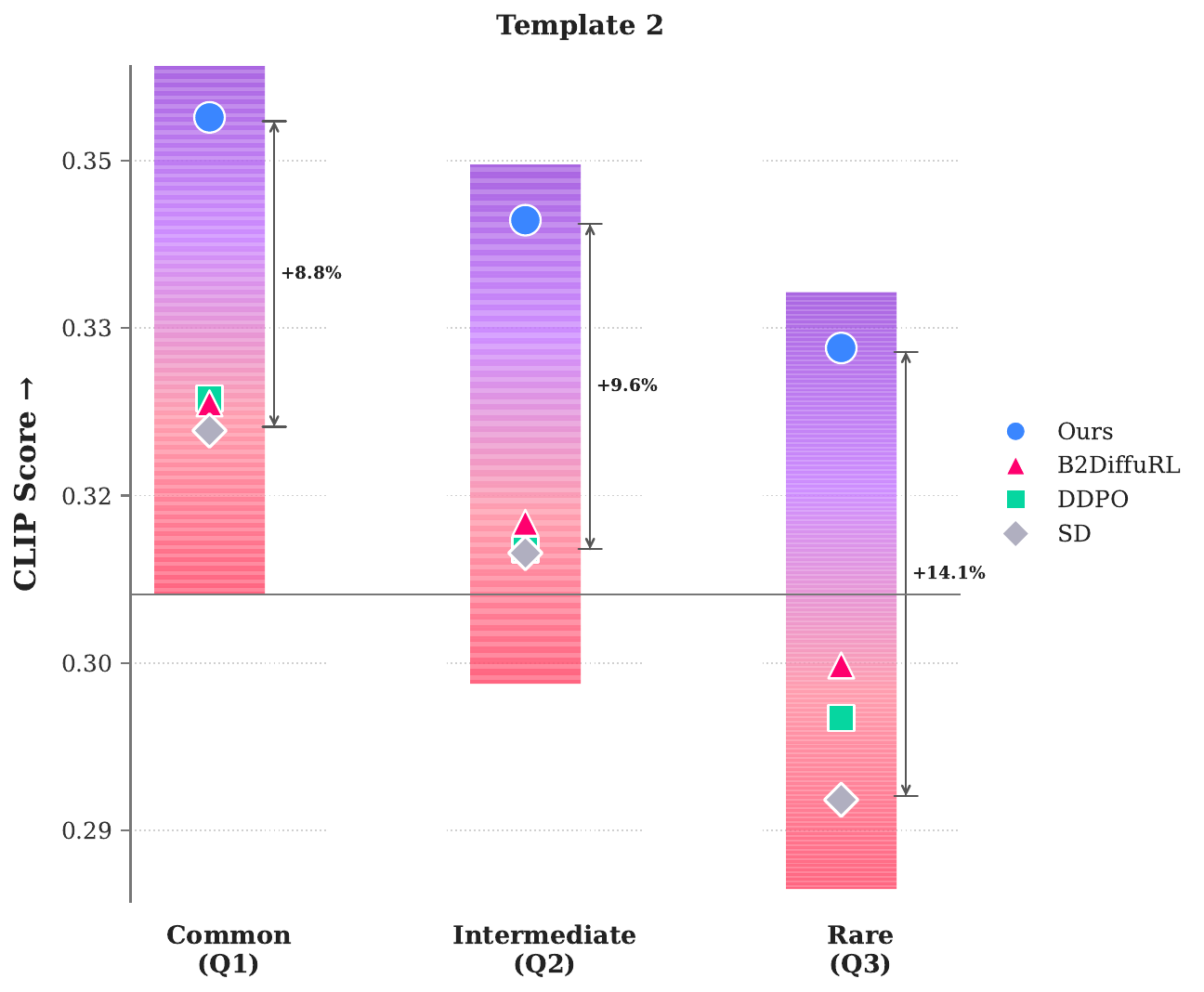}
        \caption{Template T2}
    \end{subfigure}
    \caption{Quartile reward plots for Template T1 and T2 scenarios.}
    \label{fig:quartile_t1_t2}
\end{figure}

\begin{figure}[h]
    \centering
    \begin{subfigure}{0.48\linewidth}
        \includegraphics[width=\linewidth]{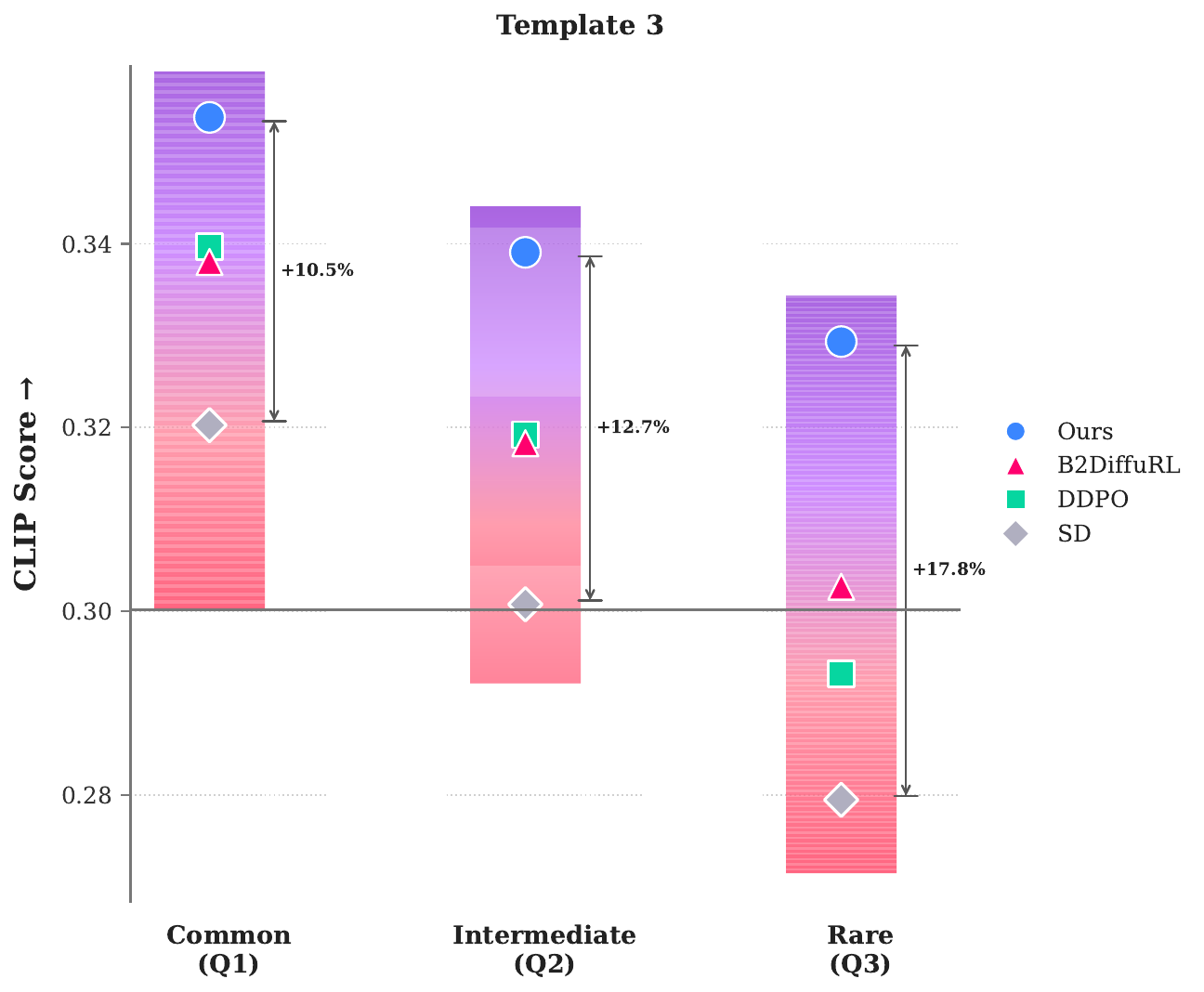}
        \caption{Template T3}
    \end{subfigure}
    \begin{subfigure}{0.48\linewidth}
        \includegraphics[width=\linewidth]{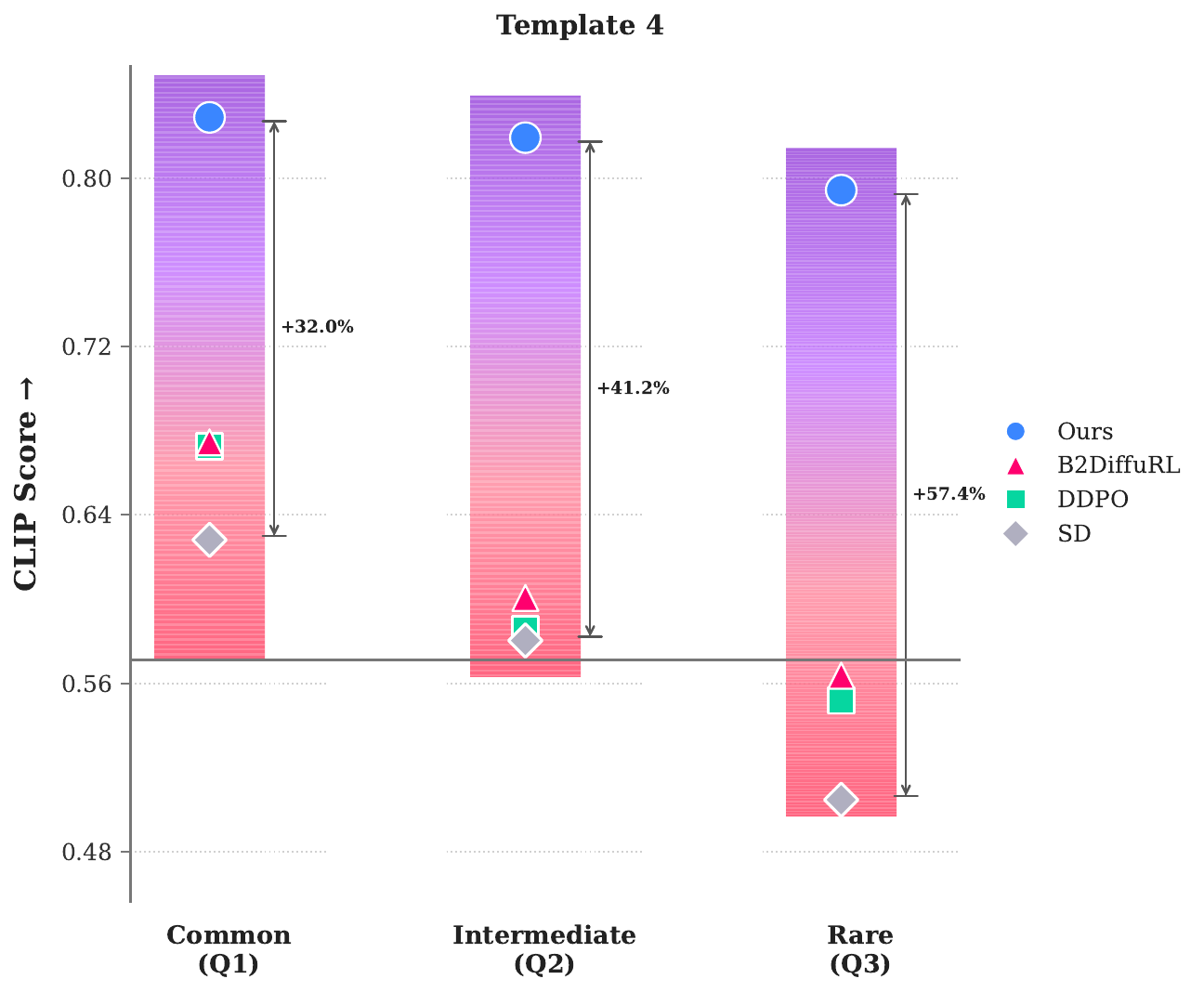}
        \caption{Vanishing Point}
    \end{subfigure}
    \caption{Quartile reward plots for Template T3 and the vanishing point geometric task.}
    \label{fig:quartile_t3_geo}
\end{figure}

\section{Additional Qualitative Results}
\label{sec:qualitative}

This section provides additional qualitative results.

\subsection{Human Evaluation Samples}


\subsection{Vanishing Point Examples}


\subsection{3D Scene Renderings}





\section{Training Time Comparison}
To compare convergence efficiency across methods, we report the wall-clock GPU time required to reach the same mean reward target (CLIP score of 0.388). This comparison isolates optimization speed under a common reward endpoint and highlights the effect of incremental training across different RL-based diffusion methods.

\begin{table}[H]
\centering
\caption{Time taken to reach the same mean reward point (0.388 CLIP score).}
\label{tab:training_time}
\begin{tabular}{lc}
\toprule
\textbf{Method} & \textbf{GPU Hours (Training Time)} \\
\midrule
DDPO & 2.55 \\
DDPO + incremental & 2.46 \\
B$^2$-DiffuRL & 3.42 \\
B$^2$-DiffuRL + incremental & 2.76 \\
FK + incremental (ours) & 9.55 \\
Branch-FK + incremental (iADD) & 7.43 \\
\bottomrule
\end{tabular}
\end{table}

\section{Limitations and Future Work}
\label{sec:limitations}

A primary limitation of our approach is its higher training cost. As reported in Table~\ref{tab:training_time}, our method requires 7.43 GPU hours, compared with 2.55 GPU hours for DDPO, to reach the same mean reward (CLIP score of 0.388). This overhead is mainly attributable to the Feynman-Kac particle-based sampling procedure, which requires repeated reward evaluations during trajectory generation. Although the incremental timestep update strategy partially reduces this cost, overall training remains slower than DDPO. Importantly, after training, inference-time cost is comparable to DDPO\cite{black2023ddpo} and does not incur the additional overhead associated with inference-time FK steering\cite{singhal2025general}.

A second limitation concerns reward estimation from noisy intermediate states. Our current approach computes rewards from denoised predictions inferred directly from noisy latents, which can be unreliable at earlier timesteps. For instance, in 3D indoor scene layout generation, improvements are limited because critical geometric attributes (e.g., object size and position) are often not accurately recoverable from noisy states, even midway through the denoising trajectory. In such settings, alternative reward-estimation schemes, as discussed in inference-time FK steering\cite{singhal2025general}, may be more suitable.

Future work should investigate more accurate estimators of terminal rewards from intermediate noisy states. In addition, stronger reward models are needed to improve prompt-image alignment. The CLIP\cite{pmlr-v139-radford21a} score used in our experiments has known failure modes; for example, for the prompt "chicken playing chess", it may assign a high score to images showing cooked chicken near a chessboard, despite weak semantic alignment. A natural next step is to extend this framework to video diffusion models, where physics-aware rewards from simulation engines could encourage temporally consistent and physically plausible video generation.

\section{Broader Impacts.}
Diffusion-based generative models are increasingly used in design, education, media creation, simulation, and human--AI interaction. A key practical limitation, however, is that high-scoring outputs do not necessarily satisfy physical or geometric rules of the real world. Our work shows that reinforcement-learning-based alignment can be trained to better respect task-specific constraints (e.g., perspective consistency or collision-free 3D layouts) while preserving output quality. This may improve reliability in applications where structural correctness matters. At the same time, more controllable generation can also be misused to produce convincing but misleading synthetic content. For this reason, safeguards, responsible deployment, and robust detection and auditing tools remain essential.


\end{document}